\documentclass[letterpaper]{article} % DO NOT CHANGE THIS
\usepackage[preprint]{aaai2027}  % DO NOT CHANGE THIS
\usepackage[hyphens]{url}  % DO NOT CHANGE THIS
\usepackage{graphicx} % DO NOT CHANGE THIS
\usepackage{natbib}  % DO NOT CHANGE THIS AND DO NOT ADD ANY OPTIONS TO IT
\usepackage{caption} % DO NOT CHANGE THIS AND DO NOT ADD ANY OPTIONS TO IT
\usepackage{algorithm}
\usepackage{algorithmic}
\usepackage{amsmath}
\usepackage{amssymb}
\usepackage{makecell}
\usepackage{subcaption}
\usepackage{xcolor} % v3: not on the disallowed-packages list; used only to color-mark new v3 edits
\usepackage{multirow}
\usepackage{newfloat}
\usepackage{listings}
\DeclareCaptionStyle{ruled}{labelfont=normalfont,labelsep=colon,strut=off} % DO NOT CHANGE THIS
\floatstyle{ruled}
\newfloat{listing}{tb}{lst}{}
\floatname{listing}{Listing}

\usepackage{booktabs}

\usepackage{amsthm}

\newtheorem{theorem}{Theorem}
\newtheorem{lemma}{Lemma}

\newtheorem{corollary}[theorem]{Corollary}

\theoremstyle{definition}

\theoremstyle{remark}
\newtheorem{remark}[theorem]{Remark}

\newcommand{\privee}{\texttt{PRIVEE}}

\title{PRIVEE: Order-Preserving Confidence Perturbation for Privacy-Preserving Vertical Federated Learning}
\title{PRIVEE: Order-Preserving Confidence Perturbation for Privacy-Preserving Vertical Federated Learning}
\author {
    Sindhuja Madabushi \textsuperscript{\rm 1},
    Haider Ali \textsuperscript{\rm 1},
    Ahmad Faraz Khan \textsuperscript{\rm 1},
    Rui Ning \textsuperscript{\rm 2},
    Hongyi Wu \textsuperscript{\rm 3},
    Chunsheng Xin \textsuperscript{\rm 2},
    Ali. R, Butt \textsuperscript{\rm 1},
    Jin-Hee Cho \textsuperscript{\rm 1},
}
\affiliations {
    \textsuperscript{\rm 1}Virginia Tech\\
    \textsuperscript{\rm 2}Old Dominion University\\
    \textsuperscript{\rm 3}University of Arizona\\
    msindhuja@vt.edu, haiderali@vt.edu, afkhan@vt.edu, rning@odu.edu, mhwu@arizona.edu, cxin@odu.edu, butta@vt.edu, jicho@vt.edu
}
\begin{document}

\maketitle

\begin{abstract}
Vertical Federated Learning (VFL) enables collaborative model training across organizations that share common user samples but hold disjoint feature spaces. Despite its potential, VFL is susceptible to feature inference attacks, in which adversarial parties exploit shared confidence scores (prediction probabilities) during inference to reconstruct private input features of other participants.  To counter this threat, we propose {\privee} (\textit{PRI}vacy-preserving \textit{V}ertical f\textit{E}derated l\textit{E}arning), a novel defense mechanism named after the French word \textit{priv\'ee}, meaning “private.” {\privee} obfuscates confidence scores while preserving critical properties such as relative ranking and inter-score distances. Rather than exposing raw scores, {\privee} only shares transformed representations, mitigating risk of reconstruction attacks without degrading model prediction accuracy.  Extensive experiments show that {\privee} achieves up to a $30\times$ increase in reconstruction error (MSE) against feature inference attacks, compared to the strongest competing defense, while preserving full predictive performance against advanced feature inference attacks.
\end{abstract}

\section{Introduction}

\paragraph{Why FL Is Vulnerable.}
Federated learning (FL)~\cite{mcmahan2017communication} enables collaborative model training without sharing raw data and has been widely adopted in privacy-sensitive domains~\cite{li2020review}. However, keeping data local does not eliminate privacy risks. Recent studies show that adversaries can exploit intermediate information exchanged during inference to reconstruct private data. In particular, in vertical federated learning (VFL), shared confidence scores enable attackers to infer other clients' private feature representations~\cite{luo2021feature, yang2023practical, jiang2022comprehensive, chen2024fia}, exposing a critical inference-time privacy vulnerability.

\paragraph{Why VFL Is Different.}
Unlike horizontal FL, where clients share the same feature space, VFL partitions features across parties that own complementary attributes of the same entities. Consequently, inference requires aggregating intermediate representations from all parties to generate confidence scores. While essential for prediction, these shared confidence scores also create an attack surface for feature inference. Additional details of \privee's model architectures are provided in Appendix~E.

\paragraph{Why Existing Defenses Fail.}
Training-time privacy techniques such as Differential Privacy (DP) and Homomorphic Encryption (HE) do not prevent inference attacks that exploit released confidence scores. Existing defenses for centralized~\cite{srivastava2014dropout, nasr2018comprehensive, yang2020defending} and VFL settings~\cite{li2023efficient, jiang2022vf, zou2022defending, lai2023vfedad} either provide limited protection or incur substantial utility loss. Even lightweight approaches, including rounding, noise injection, and DP-based perturbation~\cite{luo2021feature, jiang2022comprehensive}, often degrade prediction accuracy. These limitations motivate \privee, an inference-time defense that substantially increases reconstruction error (MSE) while preserving prediction accuracy and computational efficiency.

We propose \privee, a privacy-enhancing VFL inference framework with the following \textbf{key contributions}:
\begin{itemize}

\item \texttt{PRIVEE} protects client privacy via an order-preserving perturbation of confidence scores with negligible overhead while preserving inference accuracy. Its two adaptive variants, \texttt{PRIVEE-U} and \texttt{PRIVEE-U+}, support diverse privacy and deployment requirements.

\item \texttt{PRIVEE} effectively defends against state-of-the-art feature inference attacks, including the Generative Regression Network Attack (GRNA) and Gradient Inversion Attack (GIA)~\cite{jiang2022comprehensive}, substantially increasing reconstruction error (MSE) across all evaluated datasets and model architectures.

\item \texttt{PRIVEE} outperforms existing defenses in the large majority of evaluated settings, achieving up to $30\times$ greater reconstruction error (MSE) than the strongest competing defense (e.g., $\approx$20--22$\times$ under GRNA on MNIST, Table~\ref{tab:grna25}; up to $\approx$60$\times$ under GIA on CIFAR-10, Appendix~C), while preserving prediction accuracy and millisecond-scale inference latency for real-time VFL applications. 

\item Unlike DP- and rounding-based defenses, \texttt{PRIVEE} preserves inter-class confidence ranking, supporting downstream tasks such as ensemble learning and knowledge distillation without sacrificing privacy.

\item Extensive ablation studies show that \texttt{PRIVEE} remains effective across varying federation scales, class cardinalities, and attack strengths, validating its robustness, scalability, and adaptability to real-world VFL deployments.

\end{itemize}


\section{Related Work}

\paragraph{Inference-Time Confidence Score Protection.} Confidence score sanitization reduces information leakage through model outputs. Prior work includes top-$k$ prediction release, confidence rounding, and temperature scaling~\cite{shokri2017membership}; dropout-based regularization~\cite{srivastava2014dropout,salem2018ml}; adversarial perturbation (MemGuard)~\cite{jia2019memguard}; and autoencoder-based confidence transformation (Purifier)~\cite{yang2020defending}. However, these methods may alter prediction rankings, reduce accuracy, require retraining, or incur additional computational overhead.

Differential privacy (DP) has primarily been studied for training~\cite{abadi2016deep,jayaraman2018distributed,shukla2025federated,adnan2025framework,wang2025fedmps,demelius2025recent}. Representative inference-time methods include OP$\lambda$~\cite{roy2022strengthening} (a DP mechanism layered on top of order-preserving encryption, distinct from the standalone OPE baseline evaluated in Section~\ref{sec:exp-setup}), the one-parameter defense~\cite{ye2022one}, and FISIP~\cite{huang2011fisip}, which attempt to balance confidence obfuscation, ranking preservation, and utility. Encryption-based approaches, including order-preserving encryption~\cite{boldyreva2009order,popa2013ideal,roche2016pope,maffei2017security} and homomorphic encryption~\cite{chatterjee2017sorting,hong2021efficient,verma2022efficient}, preserve ordering or support secure computation but often incur substantial computational and communication overhead. In contrast, \texttt{PRIVEE} employs a lightweight rank-aware transformation that preserves prediction ordering without expensive cryptographic operations.

\paragraph{Defenses Against Feature Inference in VFL.}
Existing VFL defenses include ranked prediction release~\cite{rassouli2022privacy}, feature-subspace recovery (RVFR)~\cite{liu2021rvfr}, hash-based representation learning (HashVFL)~\cite{qiu2022all}, and gradient-alignment methods (FLSG)~\cite{fan2023flsg}. Although these approaches mitigate leakage at different stages of the VFL pipeline, none simultaneously provides exact prediction-order preservation, low inference-time overhead, and strong resistance to confidence-based feature inference. \texttt{PRIVEE} addresses this gap by perturbing confidence scores during inference while preserving both class ranking and prediction accuracy.

\section{Preliminaries}
\label{sec:preliminaries}

\paragraph{Threat Model.}
We adopt the canonical white-box active-party threat model of~\cite{jiang2022comprehensive} and extend it from the standard two-party setting to an $N$-party VFL system with a trusted coordinator. The active party $P_{\mathrm{act}}$ owns features $\mathbf{x}_{\mathrm{act}}$ and labels, while passive parties $P_{\mathrm{pas},i}$ hold complementary features $\mathbf{x}_{\mathrm{pas},i}$ ($i=1,\ldots,N-1$). Each party computes local logits, which the trusted coordinator aggregates into the confidence vector $\mathbf{c}$. A complete description of our two-party VFL training and inference workflow is provided in Appendix~D.

The adversary resides in $P_{\mathrm{act}}$ and has white-box access to the joint VFL model, its own features $\mathbf{x}_{\mathrm{act}}$, and the released confidence vector. Its goal is to reconstruct the private features $\mathbf{x}_{\mathrm{pas},i}$ of one or more passive parties, enabling direct comparison with established feature-inference benchmarks~\cite{jiang2022comprehensive}.

Extending feature inference to multi-party VFL introduces additional leakage from multiple passive parties and potential collusion among participants~\cite{jiang2022comprehensive}. We assume the trusted coordinator is honest and releases only perturbed confidence scores. Since \texttt{PRIVEE} operates as an inference-time post-processing layer on confidence scores, it is model-agnostic and applicable to confidence-based feature inference attacks. We evaluate \texttt{PRIVEE} against the two state-of-the-art attacks, GRNA and GIA~\cite{jiang2022comprehensive}; details of the evaluated attacks are provided in Appendix~B.

\subsection{Problem Statement}
\label{sec:problem-statement}

We formulate the defense objective as maximizing the adversary's reconstruction error while limiting the degradation in the original model's predictive performance:

\begin{equation}
\label{eq:privacy-utility-objective}
\begin{aligned}
\underset{\theta}{\operatorname{maximize}}
&\quad
\operatorname{MSE}\!\left(
d_{\mathrm{rec}}(\theta),d_{\mathrm{true}}
\right) \\
\operatorname{subject\ to}
&\quad
\left|
\mathcal{A}_{D}(\theta)-\mathcal{A}_{ND}
\right|
\leq \epsilon,
\end{aligned}
\end{equation}
where \(\theta\) denotes the defense parameters, \(d_{\mathrm{rec}}(\theta)\) is what the adversary reconstructs from the private data after applying the defense, and \(d_{\mathrm{true}}\) is the corresponding ground-truth sensitive data. The quantities \(\mathcal{A}_{D}(\theta)\) and \(\mathcal{A}_{ND}\) denote the VFL model's prediction accuracy with and without the defense, respectively. The parameter \(\epsilon\geq0\), for example \(\epsilon=0.01\), bounds the allowable accuracy degradation.

A larger reconstruction MSE indicates that the recovered private features deviate further from the true features and therefore corresponds to stronger protection against feature-inference attacks. The accuracy constraint ensures that these privacy gains do not substantially reduce the predictive utility of the underlying VFL model, yielding an explicit privacy--utility trade-off suitable for practical deployment.

\section{Proposed Defense: \privee} \label{sec:privee}

\begin{figure*}[t]
    \centering
    \includegraphics[width=\textwidth]{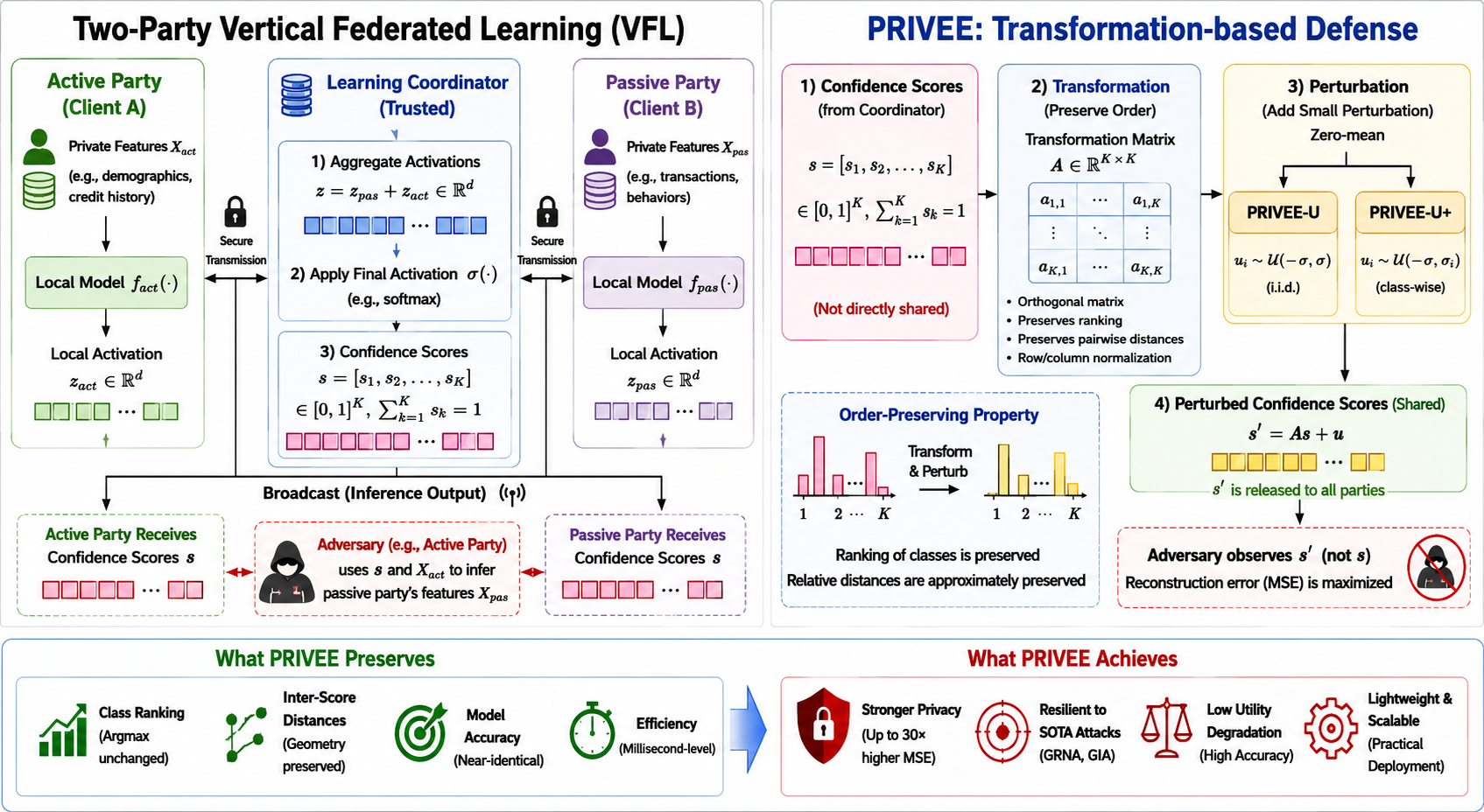}
\caption{\textbf{Overview of the proposed \texttt{PRIVEE} inference pipeline in two-party VFL.}
The coordinator applies an order-preserving confidence transformation followed by rank-aware perturbation before releasing inference outputs, preserving prediction correctness while mitigating feature inference attacks.}
    \label{fig:VFL-Inf-Defense}
\end{figure*}

\subsection{Overview of PRIVEE}

Figure~\ref{fig:VFL-Inf-Defense} illustrates the proposed \texttt{PRIVEE} framework for secure inference in two-party VFL. During inference, the active and passive parties independently compute local activations from their private features and transmit only these intermediate representations to a trusted coordinator, which aggregates them into a confidence vector and applies the
\texttt{PRIVEE} transformation, an order-preserving confidence transformation followed by rank-aware perturbation, before releasing the output. This preserves the predicted class while reducing leakage to confidence-based feature inference attacks. The framework includes two variants, \texttt{PRIVEE-U} and \texttt{PRIVEE-U+}, offering different perturbation strategies with negligible inference overhead.

\subsection{Privacy-Preserving Confidence Score Transformation}

In VFL, both active and passive parties utilize confidence score vectors to rank classes by their relative importance and estimate the likelihood of each class. Preserving the \textit{ranking order} of these scores is crucial for maintaining inference accuracy; however, retaining their absolute distances is not always necessary. To address this, our proposed defense transforms the confidence outputs to obscure class magnitudes while preserving their sorted order, thereby safeguarding privacy without sacrificing model utility.

We introduce two defense mechanisms that strike a trade-off between privacy protection, mean squared error (MSE), and inference accuracy. \textbf{1)} \texttt{PRIVEE-U}: Preserves the class ranking while perturbing confidence scores using rank-aware uniform noise, governed by perturbation parameter $\rho$; \textbf{2)} \texttt{PRIVEE-U+}: Extends \texttt{PRIVEE-U} by assigning per-class perturbation parameters $\rho_j$, offering finer-grained magnitude protection for classes with varying confidence profiles.

While prior defenses such as \textit{Purifier}~\cite{yang2020defending} and \textit{Rounding}~\cite{shokri2017membership} partially obscure confidence scores, they fail to simultaneously preserve inference accuracy. In contrast, \texttt{PRIVEE} maintains class rankings and predictive performance under strong privacy guarantees.

\subsection{Order-Preserving Perturbation}
Let $K$ denote the number of classes. We apply a transformation based on a family of orthonormal matrices $\mathbf{A} \in \mathbb{R}^{K \times K}$ as the foundation for our perturbation mechanism. Valid choices for $\mathbf{A}$ include the identity matrix $\mathbf{I}_K \in \mathbb{R}^{K \times K}$ and the negative of the first- and second-order sum and inner product-preserving (FISIP) matrix~\cite{huang2011fisip}.
Let the vector of confidence scores be denoted by $\mathbf{c} = [c_1, c_2, \ldots, c_K]^\top \in \mathbb{R}^K$. Let $\mathbf{r} \in \{1, 2, \ldots, K\}^K$ denote the ranking vector, where $r_i < r_j$ whenever $c_i > c_j$, and let $\mathbf{d} = [d_1, d_2, \ldots, d_K]^\top \in \mathbb{R}^K$ denote the transformed confidence scores that preserve relative distances.

The ranking vector $\mathbf{r}$ is computed as:
\begin{equation}
\mathbf{r} = \mathrm{rank}(\mathbf{c}) := \mathrm{argsort}(\mathrm{argsort}(-\mathbf{c})) + \mathbf{1}_K,
\end{equation}
where $\mathbf{1}_K$ is the $K$-dimensional vector of ones. This formulation ensures that higher confidence values are assigned lower rank indices, aligning with conventional ranking semantics.

We compute an initial vector of transformed confidence scores $\mathbf{d} \in \mathbb{R}^K$ using the following linear transformation $\mathbf{d} = \mathbf{A} \mathbf{c}$, 
where $\mathbf{A} \in \mathbb{R}^{K \times K}$ is an orthonormal matrix. This transformation preserves both pairwise distances and correlations. In particular, for specific choices of $\mathbf{A}$, such as the identity matrix $\mathbf{I}_K$ or the negative of the first- and second-order sum and inner product-preserving (FISIP) matrix~\cite{huang2011fisip}, ranking can also be preserved. One such example is:

\begin{equation}
\mathbf{A} = \mathbf{I}_K - \frac{2}{K} \mathbf{1}_K \mathbf{1}_K^\top.
\end{equation}

\begin{lemma}
\label{lem:distance-corr-preserving}
If $\mathbf{A} \in \mathbb{R}^{K \times K}$ is an orthonormal matrix, then the linear transformation $\mathbf{c} \mapsto \mathbf{A}\mathbf{c}$ is both \textbf{distance-preserving} and \textbf{correlation-preserving}.
\end{lemma}
\begin{proof}
See Appendix~A for the proof.
\end{proof}

The proposed method, \privee, obfuscates the raw confidence scores while preserving their order and maintaining downstream inference accuracy.
\texttt{PRIVEE-U} applies rank-aware uniform perturbation to the matrix $\mathbf{A}$
to protect confidence score magnitudes while preserving their ranking. Each diagonal
entry of $\mathbf{A}$ is perturbed according to a rank-assigned sub-interval draw
$u_j \sim \mathrm{Uniform}(I_{k_j})$, scaled by perturbation parameter $\rho$, giving
$A^{\mathrm{pert}}_{jj} = A_{jj} + u_j \cdot \sigma$ where $\sigma = C/\rho$ and $C$ is some constant.
The rank-aware assignment of subintervals guarantees that the transformation preserves the ordering of the confidence scores exactly. As an extension of \texttt{PRIVEE-U}, \texttt{PRIVEE-U+} replaces the shared perturbation scale $\sigma$ with class-specific perturbation scales $\sigma_j$. This modification is particularly beneficial for datasets with many classes, where using a single shared value of $\sigma$ may produce insufficient or uneven perturbation across classes, potentially creating privacy vulnerabilities.

\begin{lemma}
\label{lem:order-preserving}
The transformation \(c \mapsto A^{\mathrm{pert}}c\) is order-preserving with probability \(1\) when \(A = I_K\) or \(A = I_K - \frac{2}{K}\mathbf{1}_K\mathbf{1}_K^{\top}\), where \(A_{jj}^{\mathrm{pert}} = A_{jj} + u_j\sigma\), \(u_j \sim \operatorname{Uniform}(I_{k_j})\), \(k_j = K+1-\operatorname{rank}(c)_j\), and \(\sigma\) is shared across all classes.
\end{lemma}

\begin{proof}
See Appendix~A for the proof.
\end{proof}

\begin{remark}
For \texttt{PRIVEE-U+}, where $\sigma$ is replaced by per-class parameters $\sigma_j = C/\rho_j$,
order preservation holds with probability 1 under the additional condition that
$\sigma_j$ is non-decreasing with $k_j$, i.e., $\sigma_i \geq \sigma_j$ whenever
$c_i > c_j$. See Appendix~A for the argument.
\end{remark}

While the unperturbed transformation $\mathbf{c} \mapsto \mathbf{A}\mathbf{c}$ preserves both distances and rankings, $\mathbf{A}$ itself is a fixed, public matrix. Therefore, an adversary who knows $\mathbf{A}$ can trivially invert $\mathbf{d} = \mathbf{A}\mathbf{c}$ to recover $\mathbf{c}$ exactly, providing no privacy protection on its own. To address this, we introduce a rank-aware multiplicative perturbation of the diagonal entries of $\mathbf{A}$, governed by per-class uniform draws
$u_j \sim \operatorname{Uniform}(I_{k_j})$.

Because each $u_j$ is drawn independently and freshly for every released confidence vector, the specific perturbed matrix $\mathbf{A}^{\mathrm{pert}}$ used for any given query is never known to an adversary, even though its general form, given in Algorithm~\ref{alg:PRIVEE-U}, is public. This prevents the naive recovery of $\mathbf{c}$ through a fixed, known inverse $\mathbf{A}^{-1}$, while the rank ordering of
$\mathbf{p} = \mathbf{A}^{\mathrm{pert}}\mathbf{c}$ is preserved exactly, as established in Lemma \ref{lem:order-preserving}. We formalize the resulting privacy guarantees in Section ~\ref{sec:theoretical-properties}.

\begin{algorithm}[tb]
\caption{\texttt{PRIVEE-U}: Rank-Aware Diagonal Perturbation}
\label{alg:PRIVEE-U}
\begin{algorithmic}[1]
  \STATE \textbf{Input:} Confidence scores $\mathbf{c} \in \mathbb{R}^K$,
  perturbation parameter $\rho > 0$, scaling constant $C > 0$, and
  matrix $\mathbf{A} \in \mathbb{R}^{K \times K}$
  
  \STATE \textbf{Output:} Perturbed confidence scores
  $\mathbf{p} \in \mathbb{R}^K$
  
  \STATE $\sigma \gets C/\rho$
  \STATE Divide $[0,1]$ into $K$ equal subintervals, denoted by
  $I_1,\ldots,I_K$
  \STATE $\mathbf{A}^{\mathrm{pert}} \gets \mathbf{A}$
  
  \FOR{$j = 1$ \TO $K$}
    \STATE $k \gets K + 1 -
    \bigl(\operatorname{rank}(\mathbf{c})\bigr)_j$
    \STATE Sample $u_j \sim \operatorname{Uniform}(I_k)$
    \STATE $A^{\mathrm{pert}}_{jj}
    \gets A_{jj} + u_j\sigma$
  \ENDFOR
  
  \STATE $\mathbf{p} \gets
  \mathbf{A}^{\mathrm{pert}}\mathbf{c}$
  \STATE \textbf{return} $\mathbf{p}$
\end{algorithmic}
\end{algorithm}

Algorithm~\ref{alg:PRIVEE-U} summarizes this rank-aware perturbation procedure in full; \texttt{PRIVEE-U+} follows the same steps with per-class scales $\sigma_j = C/\rho_j$ in place of the shared $\sigma$.

\begin{table}[t]
\centering
\scriptsize
\setlength{\tabcolsep}{6pt}
\renewcommand{\arraystretch}{1.15}
\caption{Computation overhead of defense methods for varying numbers of classes. Lower runtime is better. Confidence scores were randomly generated, and all methods were evaluated independently of model training.}
\label{tab:defense_times}

\begin{tabular}{cccc}
\hline
\textbf{\# Classes} &
\textbf{\texttt{PRIVEE}} &
\textbf{OPE} &
\textbf{Top-5 HE} \\
\hline
10     & \textbf{1.6 ms}   & 17.2 ms & 7 min  \\
100    & \textbf{3.8 ms}   & 52.5 ms & 25 min \\
1,000  & \textbf{12.0 ms}  & 463.5 ms & 40 min \\
10,000 & \textbf{145.4 ms} & 6.95 s & 106 min \\
\hline
\end{tabular}
\end{table}

\section{Theoretical Properties of \privee}
\label{sec:theoretical-properties}
We provide a formal characterization of what \texttt{PRIVEE-U} preserves, what privacy it guarantees for a single released confidence vector, and formalization of its protection under repeated queries. This guarantee is a distinct, ambiguity-based notion of privacy and is not a formal $(\varepsilon,\delta)$-differential-privacy guarantee; Section~\ref{sec:conclusion-future-work} discusses this distinction as an explicit limitation.
We prove that adding the rank-aware diagonal perturbation does not preserve distances and correlations exactly, because the perturbed matrix is generally no longer orthonormal. 
However, the deviation is deterministically bounded by $\epsilon(\sigma)=2\sigma+\sigma^2$. When $\epsilon(\sigma)<1$, the transformed Euclidean distance remains within explicit multiplicative lower and upper bounds, and the corresponding normalized inner-product correlation distortion is also bounded.
Thus, the theoretically valid utility claim is bounded, not exact preservation under perturbation. Second, we analyze a single released vector. Because each unknown perturbation draw $u_j$ lies in a known rank-dependent interval, every original confidence score $c_j$ can be bounded within a reconstruction interval. More importantly, we show that, with probability 1, the same released vector is consistent with an uncountable continuum of distinct valid confidence vectors. Therefore, a single release does not uniquely identify the original confidence magnitudes, although it still reveals their ordering because rank is preserved. Finally, we study repeated queries of the same input. We derive the exact probability that a minimum-based estimator reconstructs each confidence coordinate within a specified error after (T) queries. This proves that \texttt{PRIVEE-U} does not provide asymptotic privacy against unlimited repeated queries. Instead, the theory supports a precise query-budget condition that bounds the probability of successful high-accuracy reconstruction. Detailed proofs of all of the theoretical properties of \texttt{PRIVEE} are in Appendix~A.

\begin{table*}[t]
\centering
\scriptsize
\setlength{\tabcolsep}{5pt}
\renewcommand{\arraystretch}{1.15}
\caption{Training Accuracy Change (\%) Under the GRN Attack (Attack Strength = 50\%). OPE, \texttt{PRIVEE-U}, and \texttt{PRIVEE-U+} are exactly $0.00$ by construction: all three transformations are provably rank-preserving (Lemma~\ref{lem:order-preserving} for \texttt{PRIVEE}) and therefore never change the predicted class.}
\label{tab:train_acc_change_50}

\begin{tabular}{lrrrrrrr}
\toprule
\textbf{Dataset} &
\textbf{R(1)} &
\textbf{R(2)} &
\textbf{DP ($\varepsilon=0.5$)} &
\textbf{DP ($\varepsilon=0.7$)} &
\textbf{OPE} &
\textbf{\texttt{PRIVEE-U}} &
\textbf{\texttt{PRIVEE-U+}} \\
\midrule
MNIST & $-0.011\pm0.003$ & $-0.009\pm0.004$ & $-63.742\pm1.185$ & $-50.517\pm0.943$ & $0.00\pm0.00$ & $0.00\pm0.00$ & $0.00\pm0.00$ \\
CIFAR10 & $0.190\pm0.052$ & $-0.020\pm0.010$ & $44.296\pm0.802$ & $38.108\pm0.871$ & $0.00\pm0.00$ & $0.00\pm0.00$ & $0.00\pm0.00$
 \\
CIFAR100 & $-0.512\pm0.024$ & $-0.046\pm0.009$ & $60.266\pm0.400$ & $56.504\pm0.526$ & $0.00\pm0.00$ & $0.00\pm0.00$ & $0.00\pm0.00$
 \\
Drive Diagnosis & $-0.648\pm0.027$ & $-0.052\pm0.011$ & $-56.214\pm0.736$ & $-41.683\pm0.612$ & $0.00\pm0.00$ & $0.00\pm0.00$ & $0.00\pm0.00$ \\
Adult Income & $+0.032\pm0.009$ & $-0.048\pm0.012$ & $-10.764\pm0.384$ & $-7.236\pm0.291$ & $0.00\pm0.00$ & $0.00\pm0.00$ & $0.00\pm0.00$ \\
\bottomrule
\end{tabular}
\end{table*}

\begin{table*}[t]
\centering
\scriptsize
\setlength{\tabcolsep}{5pt}
\renewcommand{\arraystretch}{1.15}
\caption{Gradient Inversion Attack (GIA) results on MNIST with attack strength $0.5$. Each entry reports the MSE before and after applying the defense (Before, After) for different numbers of clients. Higher post-defense MSE indicates stronger protection against gradient inversion.}
\label{tab:mnist-gia-as05}

\begin{tabular}{lccccc}
\hline
\textbf{Defense} & \textbf{5 Clients} & \textbf{10 Clients} & \textbf{15 Clients} & \textbf{20 Clients} & \textbf{25 Clients} \\
\hline
R(1)
& 0.1097, 0.1102
& 0.0466, 0.1097
& 0.0916, 0.0920
& 0.0910, 0.0914
& 0.0899, 0.0902 \\

R(2)
& 0.1051, 0.1052
& 0.0944, 0.0945
& 0.0927, 0.0927
& 0.0956, 0.0956
& 0.0915, 0.0916 \\

OPE
& 0.1054, 0.3012
& 0.0985, 0.2978
& 0.0977, 0.2947
& 0.1022, 0.3026
& 0.0935, 0.2927 \\

DP ($\varepsilon=0.5$)
& 0.1009, 0.9680
& 0.0999, 0.9631
& 0.0915, 0.9461
& 0.0901, 0.9438
& 0.0887, 0.9410 \\

DP ($\varepsilon=1$)
& 0.0848, 0.2911
& 0.1031, 0.3147
& 0.0994, 0.3094
& 0.0975, 0.3066
& 0.0977, 0.3069 \\

\texttt{PRIVEE-U} ($\rho=0.1$)
& 0.1097, 19.0200
& 0.0950, 20.2756
& 0.0961, 19.3819
& 0.0947, 20.3509
& 0.0850, 20.5965 \\

\texttt{PRIVEE-U+}
& 0.1042, 2.6439
& 0.0969, 2.4721
& 0.0980, 2.4146
& 0.0887, 2.3652
& 0.0888, 2.4378 \\
\hline
\end{tabular}
\end{table*}

\begin{table*}[t]
\centering
\scriptsize
\setlength{\tabcolsep}{5pt}
\renewcommand{\arraystretch}{1.15}
\caption{MSE of the GRN attack under different defense mechanisms across datasets with attack strength $25\%$. Higher MSE indicates stronger resistance to gradient reconstruction attacks.}
\label{tab:grna25}

\begin{tabular}{lrrrrrrrr}
\hline
\textbf{Dataset} & \textbf{No Defense} & \textbf{R(1)} & \textbf{R(2)} & \textbf{OPE} & \textbf{DP ($\varepsilon=0.5$)} &
\textbf{DP ($\varepsilon=0.7$)} &
\textbf{\texttt{PRIVEE-U}} &
\textbf{\texttt{PRIVEE-U+}} \\
\hline
MNIST
& $0.062\pm0.009$ & $0.059\pm0.003$ & $0.063\pm0.004$ & $0.280\pm0.003$ & $0.896\pm0.001$ & $0.480\pm0.000$ & $\textbf{19.070}\pm \textbf{1.091}$ & $2.715\pm0.000$ \\

CIFAR100
& $0.140\pm0.013$ & $0.013 \pm 0.000$ & $0.013 \pm 0.00$ & $0.329 \pm 0.00$ & $0.944 \pm 0.000$ & $0.4876 \pm 0.0006$ & $1.693 \pm 0.007$ & $\textbf{2.652} \pm \textbf{0.097}$\\

CIFAR10
& $0.1308 \pm 0.0031$ & $0.126 \pm 0.003$ & $0.126 \pm 0.002$ & $0.374 \pm 0.003$ & $0.990 \pm 0.004$ & $0.559 \pm 0.001$ & $\textbf{11.522} \pm \textbf{0.582}$ & $1.482 \pm 0.051$ \\

Drive Diagnosis
& $0.1390\pm0.007$ & $0.151\pm0.017$ & $0.156\pm0.009$ & $0.381\pm0.001$ & $1.0818\pm0.001$ & $0.621\pm0.004$ & $\textbf{19.25}\pm \textbf{0.091}$ & $2.704\pm0.087$ \\

Adult Income
& $0.229\pm0.006$ & $0.232\pm0.009$ & $0.276\pm0.011$ & $0.264\pm0.007$ & $1.181\pm0.018$ & $0.724\pm0.013$ & $\textbf{5.498}\pm\textbf{0.052}$ & $\textbf{4.829}\pm\textbf{0.046}$\\
\hline
\end{tabular}
\end{table*}
   
\begin{table*}[t]
\centering
\scriptsize
\setlength{\tabcolsep}{5pt}
\renewcommand{\arraystretch}{1.15}
\caption{MSE of the GRN attack under different defense mechanisms across datasets with attack strength $50\%$. Higher MSE indicates stronger resistance to gradient reconstruction attacks.}
\label{tab:grna50}

\begin{tabular}{lrrrrrrrr}
\hline
\textbf{Dataset} &
\textbf{No Defense} &
\textbf{R(1)} &
\textbf{R(2)} &
\textbf{OPE} &
\textbf{DP ($\varepsilon=0.5$)} &
\textbf{DP ($\varepsilon=0.7$)} &
\textbf{\texttt{PRIVEE-U}} &
\textbf{\texttt{PRIVEE-U+}} \\
\hline
MNIST
& $0.103\pm0.003$ & $0.106\pm0.004$ & $0.103\pm0.006$ & $0.307\pm0.008$ & $0.965\pm0.014$ & $0.541\pm0.010$ & $\mathbf{20.491\pm1.145}$ & $2.503\pm 0.039$ \\

CIFAR100
& $0.012\pm0.000$ & $0.013\pm0.000$ & $0.012\pm0.000$ & $0.329\pm0.000$ & $0.943\pm0.006$ & $0.487\pm0.004$ & $1.568\pm0.008$ & $\mathbf{2.381\pm0.228}$ \\

CIFAR10
& $0.124\pm0.000$ & $0.126\pm0.000$ & $0.124\pm0.000$ & $0.362\pm0.000$ & $0.955\pm0.003$ & $0.534\pm0.001$ & $\mathbf{10.874\pm0.649}$ & $1.534\pm0.116$ \\

Drive Diagnosis
& $0.133\pm0.006$ & $0.143\pm0.009$ & $0.149\pm0.008$ & $0.348\pm0.007$ & $1.019\pm0.015$ & $0.606\pm0.011$ & $\mathbf{16.074\pm0.084}$ & $2.401\pm0.043$ \\

Adult Income
& $0.232\pm0.008$ & $0.229\pm0.007$ & $0.263\pm0.010$ & $0.274\pm0.009$ & $1.117\pm0.019$ & $0.761\pm0.014$ & $\mathbf{6.548\pm0.061}$ & $3.908\pm0.052$ \\
\hline
\end{tabular}
\end{table*}

\section{Experimental Setup} \label{sec:exp-setup}

\paragraph{Datasets.}
We evaluate \texttt{PRIVEE} on five benchmarks: two tabular datasets, Drive Diagnosis~\cite{dua2017} and ADULT Income~\cite{adult_2}, and three image datasets, MNIST~\cite{deng2012mnist}, CIFAR-10, and CIFAR-100~\cite{krizhevsky2009learning}. These datasets cover binary and multiclass tasks with up to 100 classes (CIFAR-100), enabling evaluation across different data modalities and class counts.

\paragraph{Metrics.}

We use the following evaluation metrics:
\begin{itemize}
    \item \textbf{Attacker’s Mean Squared Error (MSE)} quantifies privacy as the MSE between reconstructed and true data:
    \begin{equation}
    \mathrm{MSE} = \frac{1}{n \cdot d_{\mathrm{target}}} \sum_{t=1}^{n} \sum_{i=1}^{d_{\mathrm{target}}} (\hat{x}^{t}_{\mathrm{target},i} - x^{t}_{\mathrm{target},i})^2,
    \end{equation}
    where $n$ is the number of samples and $d_{\mathrm{target}}$ the number of target features. Higher MSE indicates better privacy.
    
    \item \textbf{Change in Accuracy (CA)} measures accuracy degradation due to the defense:
    \begin{equation}
    \Delta \mathcal{A} = \mathcal{A}_{D} - \mathcal{A}_{ND},
    \end{equation}
    where $\mathcal{A}_{D}$ and $\mathcal{A}_{ND}$ denote accuracies with and without the defense.

\end{itemize}

\paragraph{Comparison with Baselines}

We compare \texttt{PRIVEE-U} and \texttt{PRIVEE-U+} against the following SOTA defenses:
\begin{itemize}
    \item \textbf{Top-$k$ Homomorphic Encryption Sorting}~\cite{verma2022efficient}: Returns encrypted sorting of the top-$k$ confidence scores ($k=5$).  In preliminary experiments, this approach incurred substantial overhead on datasets with tens of thousands of samples, with each training run taking approximately 20 minutes per epoch. As a result, we excluded it from the full attack evaluation. While Top-$k$ Homomorphic Encryption, like OPE, preserves model accuracy and protects confidence values, it is not computationally lightweight.

    \item \textbf{Order-Preserving Encryption (OPE)}~\cite{maffei2017security}: Obfuscates duplicate frequencies while preserving score order.

    \item \textbf{Rounding}~\cite{shokri2017membership}: Reduces precision of confidence scores to $b$ floating-point digits; we report two precisions, R(1) and R(2), corresponding to $b=1$ and $b=2$ decimal digits, respectively.

    \item \textbf{Differential Privacy (DP)}~\cite{holohan2019diffprivlib}: Implements the Gaussian mechanism via IBM’s \texttt{diffprivlib} to ensure formal privacy guarantees; 
% the $\varepsilon$ values swept per attack family are listed in Appendix~E.
\end{itemize}

Full hyperparameter settings, model architectures, random-seed configuration, and computing infrastructure for all methods are provided in Appendix~E.

\section{Results \& Analyses}
\label{sec:results-analyses}

\subsection{Effect of Attacks on \texttt{PRIVEE}}

\paragraph{\bf GRN Attack.}
Tables \ref{tab:grna25} and \ref{tab:grna50} show the effect of the GRN attack across datasets and baseline defenses at 25\% and 50\% attack strength (Appendix~C reports the 75\%-strength results; the pattern is broadly consistent, with above results). Across all datasets, the \textit{No Defense} setting yields very low reconstruction error, confirming that the GRN attack can nearly perfectly recover passive-party data from raw confidence scores. Adding standard differential privacy (DP) noise increases the MSE. For example, DP($\varepsilon = 0.5$) roughly doubles or triples the reconstruction error compared to DP($\varepsilon = 0.7$). In contrast, \texttt{PRIVEE-U} increases the MSE to the tens for simpler datasets (e.g., MNIST) and the single digits for most other datasets, substantially outperforming standard DP.

With $K=100$ classes, the $K$ equal-width sub-intervals from which \texttt{PRIVEE-U}'s shared perturbation scale $\sigma$ draws (Algorithm~\ref{alg:PRIVEE-U}) are an order of magnitude narrower than for the lower-class-count datasets, leaving less room to perturb any single class. \texttt{PRIVEE-U+}'s per-class scales $\sigma_j$ are not subject to this constraint and remain the strongest defense for CIFAR-100 across all three attack strengths (Tables~\ref{tab:grna25}--\ref{tab:grna50}, Appendix~C), making it the recommended variant for large-$K$ deployments.

These results indicate that, at 25\% attack strength, \texttt{PRIVEE‐U} provides robust defense across diverse data types, while \texttt{PRIVEE‐U+} offers additional benefits for mid-complexity image datasets. Effect if GIA attack on \privee is given in Appendix~C. The rationale for selecting the perturbation-amplitude constant
$C=0.48$, together with details of the corresponding hyperparameter
grid search, is provided in Appendix~E.

\subsection{Effect of Increasing Number of Classes}

While all baselines perform efficiently on a 10-class problem, encryption-based methods become increasingly inefficient as class count grows (Table~\ref{tab:defense_times}). \texttt{PRIVEE-U} scales near-linearly, with runtime increasing modestly from 0.0016~sec. at 10 classes to 0.1454~sec. at 10,000 classes, making it well suited for large-scale applications.  This makes \texttt{PRIVEE-U} the preferred choice for large-scale problems: it is lightweight, accuracy-preserving, tunably private, and remains highly efficient across class sizes.

\subsection{Accuracy Analysis of \texttt{PRIVEE}}

Table~\ref{tab:train_acc_change_50} reports the change in training accuracy under the GRN attack with 50\% attack strength (baseline training and inference accuracies prior to any attack or defense are reported in Appendix~C). Rounding causes negligible accuracy degradation, whereas applying Differential Privacy (DP) directly to confidence scores substantially reduces accuracy. Although OPE preserves accuracy, its computational overhead grows rapidly with the number of classes, limiting scalability. In contrast, both \texttt{PRIVEE-U} and \texttt{PRIVEE-U+} incur no accuracy loss while providing encryption-level privacy. Unlike OPE, \texttt{PRIVEE} is lightweight and scalable, making it well suited for practical VFL deployments.

\subsection{Ablation Studies}

\begin{table}[t]
\centering
\scriptsize
\setlength{\tabcolsep}{1.5pt}
\renewcommand{\arraystretch}{1.15}
\caption{Effect of the number of clients and $\rho$ on GRN attack performance for MNIST using \texttt{PRIVEE-U}. Higher MSE indicates stronger resistance to gradient reconstruction. All configurations incur 0\% accuracy loss.}
\label{tab:mnist-grna-combined}

\begin{subtable}[t]{0.48\columnwidth}
\centering
\caption{Effect of the number of clients ($\rho=0.1$).}
\label{tab:mnist-grna01}

\begin{tabular}{cccc}
\hline
\rule{0pt}{2.5ex}
\textbf{\# Clients} &
\shortstack{\textbf{MSE}\\\textbf{Without}} &
\shortstack{\textbf{MSE}\\\textbf{With}} &
\shortstack{\textbf{Final}\\\textbf{Acc.}}\\
\hline
5  & 5.4242 & 21.1015 & 93.44 \\
10 & 4.7308 & 19.2787 & 93.77 \\
15 & 3.7724 & 17.1672 & 93.99 \\
20 & 3.9231 & 17.2209 & 94.22 \\
25 & 4.2270 & 18.0560 & 94.23 \\
\hline
\end{tabular}
\end{subtable}
\hfill
\begin{subtable}[t]{0.48\columnwidth}
\centering
\caption{Effect of $\rho$ with 25 clients.}
\label{tab:mnist-grna25}

\begin{tabular}{cccc}
\hline
\rule{0pt}{2.5ex}
\textbf{$\rho$} &
\shortstack{\textbf{MSE}\\\textbf{Without}} &
\shortstack{\textbf{MSE}\\\textbf{With}} &
\shortstack{\textbf{Final}\\\textbf{Acc.}}\\
\hline
0.05 & 4.2216 & 57.7185 & 94.03 \\
0.07 & 4.0732 & 26.8967 & 94.27 \\
0.10 & 3.8712 &  5.4833 & 94.17 \\
0.30 & 4.1242 &  5.8225 & 94.19 \\
0.50 & 4.2944 &  4.9572 & 94.17 \\
0.90 & 4.0411 &  4.1987 & 94.20 \\
\hline
\end{tabular}
\end{subtable}

\end{table}
Table~\ref{tab:mnist-gia-as05} reports MSE before and after applying each defense across baselines with varying numbers of clients. The results indicate that client count does not significantly affect relative MSE trends, as the exposure of confidence vectors remains the same regardless of federation size. Consequently, MSE before defense is consistent across client counts, while MSE after defense reflects only the effectiveness of the chosen defense. Since confidence scores are broadcast to all clients, the potential leakage of sensitive information is uniform, implying that scaling up the federation does not alter inference risks; only defenses explicitly designed to suppress leakage can provide meaningful protection.

To further evaluate the robustness of \texttt{PRIVEE}, we conducted experiments varying both the number of clients (Table~\ref{tab:mnist-grna01}) and the perturbation parameter $\rho$ (Table~\ref{tab:mnist-grna25}). The results provide two complementary insights.

First, with fixed $\rho$ (e.g., $\rho=0.1$ in Table~\ref{tab:mnist-grna01}), scaling the federation from 5 to 25 clients does not affect MSE (before or after defense) or accuracy. Since confidence vectors are broadcast to all clients, the attack surface remains unchanged regardless of federation size. Thus, \texttt{PRIVEE} ensures stable privacy guarantees as the number of participants grows, while maintaining model utility.

Second, with a fixed number of clients (25 in Table~\ref{tab:mnist-grna25}) and varying $\rho$, we observe the expected privacy–utility tradeoff. Smaller $\rho$ values yield larger gaps between MSE before and after defense, indicating stronger protection against GRN attacks, with no accuracy loss, the $\Delta$ accuracy is consistently zero, and final accuracy remains above $90\%$. As $\rho$ increases, protection weakens, but utility remains stable. 

% \sindhuja{Experiments for scalability with varying $\rho$ values and number of clients are given in Appendix~C}

To recap, these findings show: (i) in our tested MNIST/GRNA setting, \texttt{PRIVEE} is unaffected by the number of clients, suggesting it addresses scalability concerns in VFL and (ii) practitioners can tune $\rho$ to strengthen privacy without compromising accuracy, with smaller $\rho$ offering the most robust protection. For \texttt{PRIVEE-U}, guarantees depend on selecting appropriate $\rho$ ranges, while \texttt{PRIVEE-U+}, which preserves ordering (Lemma \ref{lem:order-preserving}) without affecting accuracy, is expected to exhibit similar trends.

\section{Conclusion \& Future Work} \label{sec:conclusion-future-work}

\paragraph{Summary.}
We introduced \texttt{PRIVEE}, a lightweight inference-time defense for VFL that obscures confidence magnitudes while preserving class rankings. Its two variants, \texttt{PRIVEE-U} and \texttt{PRIVEE-U+}, provide different levels of perturbation control.
Theoretically, \texttt{PRIVEE} preserves the complete ranking with probability one, bounds distortion in distances and correlations, and makes each released vector consistent with uncountably many possible originals. However, repeated queries can reduce adversarial uncertainty, motivating deployment query limits.
Empirically, \texttt{PRIVEE} substantially increased feature-reconstruction error across datasets, architectures, federation settings, and attack strengths while fully preserving model accuracy. It incurred only millisecond-scale overhead and avoided the privacy--utility trade-offs of conventional confidence-score DP.

\paragraph{Limitations.} \texttt{PRIVEE}'s guarantee is an ambiguity/reconstruction-hardness bound (Appendix~A), not a formal $(\varepsilon,\delta)$-differential-privacy guarantee, and should not be substituted for DP where a provable guarantee is required. The mechanism assumes an honest, non-colluding coordinator; if the coordinator is compromised or colludes with the active party, \texttt{PRIVEE} provides no protection because the perturbation is applied by the coordinator itself. Repeated queries on the same input reduce protection over time (Appendix~A gives an exact convergence rate and a query-budget bound for \texttt{PRIVEE-U}); although the theory bounds this degradation, we have not yet empirically validated the bound or extended it to \texttt{PRIVEE-U+}. Finally, our evaluation covers five datasets (up to 100 classes) and two attacks (GRNA and GIA); the client-count and $\rho$ ablation (Section~\ref{sec:results-analyses}) is demonstrated on MNIST/GRNA only, and generalization of the ``unaffected by client count'' claim to other datasets and attacks remains to be validated. \texttt{PRIVEE-U}'s protection for CIFAR-100 (100 classes) is also comparatively weaker at high attack strength than for lower-class-count datasets, since its shared perturbation scale has less room to act across $K$ narrow sub-intervals; \texttt{PRIVEE-U+} is recommended for such large-$K$ settings.

\paragraph{Future Work.} Future work will strengthen \texttt{PRIVEE} in three directions. We will investigate structured perturbation mechanisms that better preserve distances, inner products, and correlations while concealing confidence magnitudes. We will extend the framework to address repeated and adaptive queries through query-aware perturbation, privacy budgets, rate limiting, and stateful defenses, while relaxing the trusted-coordinator assumption by considering collusion and untrusted coordinators. Finally, we will evaluate \texttt{PRIVEE} on larger VFL systems, additional model architectures, more diverse attacks and datasets, and tasks requiring calibrated confidence scores.

\section*{Ethical Statement}

This work aims to improve privacy protection in vertical federated learning (VFL) by reducing information leakage from confidence scores during inference. The proposed framework is intended for legitimate applications involving privacy-sensitive data, including healthcare, finance, and IoT systems. Although our evaluation includes feature inference attacks such as GRNA and GIA, these attacks are studied solely to assess defensive effectiveness and strengthen the security of VFL systems. We do not introduce new attack techniques or provide tools intended to facilitate malicious exploitation. While \texttt{PRIVEE} substantially increases resistance to inference attacks, it should not be considered a complete replacement for complementary security mechanisms such as secure communication, authentication, access control, and rigorous system validation. Responsible deployment should combine \texttt{PRIVEE} with established security and privacy best practices.

% \bibliography{aaai2027}
\clearpage

\twocolumn[
\begin{center}
    {\LARGE\bfseries
    Supplementary Material for PRIVEE:\\[0.3em]
    Privacy-Preserving Vertical Federated Learning Against
    Feature Inference Attacks
    \par}
\end{center}

\vspace{1.5em}
]

\vspace{1em}

\section{Formal Privacy Guarantees}
\label{sec:formal-privacy-guarantees}

We establish the theoretical properties of \texttt{PRIVEE-U} in three parts: the extent to which the perturbed transformation preserves distance and correlation, the ambiguity guaranteed for any single released confidence vector, and an exact, closed-form characterization of what an adversary can and cannot recover under repeated queries on the same input.

\subsection{Proofs of Main-Text Lemmas}
\label{sec:main-text-proofs}

We restate and prove the two lemmas from the main paper's Order-Preserving Perturbation section that are used throughout the rest of this appendix.

\textbf{Lemma 1 (Distance- and Correlation-Preservation).} If $\mathbf{A} \in \mathbb{R}^{K \times K}$ is an orthonormal matrix, then the linear transformation $\mathbf{c} \mapsto \mathbf{A}\mathbf{c}$ is both distance-preserving and correlation-preserving.

\begin{proof}
Since $\mathbf{A}$ is orthonormal, we have $\mathbf{A}^\top \mathbf{A} = \mathbf{I}_K$. For any vectors $\mathbf{c}, \mathbf{c}' \in \mathbb{R}^K$,
$\lVert \mathbf{A}\mathbf{c} - \mathbf{A}\mathbf{c}' \rVert_2
= \lVert \mathbf{A}(\mathbf{c} - \mathbf{c}') \rVert_2
= \lVert \mathbf{c} - \mathbf{c}' \rVert_2,$
establishing distance preservation.
Similarly, the inner product is invariant under orthonormal transformation:
\begin{equation*}
\langle \mathbf{A}\mathbf{c}, \mathbf{A}\mathbf{c}' \rangle
= (\mathbf{A}\mathbf{c})^\top (\mathbf{A}\mathbf{c}')
= \mathbf{c}^\top \mathbf{A}^\top \mathbf{A} \mathbf{c}'
= \langle \mathbf{c}, \mathbf{c}' \rangle,
\end{equation*}
which implies $\mathrm{corr}(\mathbf{A}\mathbf{c}, \mathbf{A}\mathbf{c}') = \mathrm{corr}(\mathbf{c}, \mathbf{c}')$. Thus, the transformation preserves both distances and correlations.
\end{proof}

\textbf{Lemma 2 (Order-Preservation).} The transformation $c \mapsto A^{\mathrm{pert}}c$ is order-preserving with probability $1$ when $A = I_K$ or $A = I_K - \frac{2}{K}\mathbf{1}_K\mathbf{1}_K^{\top}$, where $A_{jj}^{\mathrm{pert}} = A_{jj} + u_j\sigma$, $u_j \sim \operatorname{Uniform}(I_{k_j})$, $k_j = K+1-\operatorname{rank}(c)_j$, and $\sigma$ is shared across all classes.

\begin{proof}
Write $A^{\mathrm{pert}} = A + \operatorname{diag}(u_1\sigma,\ldots,u_K\sigma)$. For $A = I_K - \frac{2}{K}\mathbf{1}_K\mathbf{1}_K^{\top}$, we have
$(Ax)_j = x_j - \frac{2}{K}\sum_{r=1}^{K}x_r$
for any $x \in \mathbb{R}^K$, while $(Ax)_j=x_j$ trivially when $A=I_K$. Hence,
$p_j = c_j(1+u_j\sigma)-\frac{2}{K}\sum_{r=1}^{K}c_r.$
Consider classes $i$ and $j$ such that $c_i>c_j>0$. By construction, $\operatorname{rank}(c)_i < \operatorname{rank}(c)_j$, so $k_i>k_j$. Since $I_{k_i}$ lies entirely above $I_{k_j}$, it follows that $u_i>u_j$ with probability $1$.

The term $\frac{2}{K}\sum_{r=1}^{K}c_r$ is identical for every class and therefore cancels in the pairwise difference:
$p_i-p_j = (c_i-c_j)+\sigma(c_i u_i-c_j u_j).$
The first term is positive by assumption. Moreover, $c_i>c_j>0$ and $u_i>u_j>0$ almost surely imply $c_i u_i>c_j u_j$. Therefore, both terms are positive, and hence $p_i>p_j$.
The same argument applies when $A=I_K$, except that there is no common additive term to cancel. Since the inequality holds for every pair satisfying $c_i>c_j$, we conclude that
$\operatorname{rank}\!\left(A^{\mathrm{pert}}c\right)=\operatorname{rank}(c)$.
\end{proof}

\textbf{Remark (\texttt{PRIVEE-U+} Order Preservation).} For \texttt{PRIVEE-U+}, where $\sigma$ is replaced by per-class parameters $\sigma_j = C/\rho_j$, order preservation holds with probability 1 under the additional condition that $\sigma_j$ is non-decreasing with $k_j$, i.e., $\sigma_i \geq \sigma_j$ whenever $c_i > c_j$. Under this condition, $u_i \sigma_i > u_j \sigma_j$ (since both $u_i > u_j$ and $\sigma_i \geq \sigma_j$), and the remainder of the proof of Lemma 2 above follows identically.

\subsection{Preservation Under Perturbation}

\textbf{Lemma 1} in the main document established that the unperturbed transformation \(\mathbf{c} \mapsto \mathbf{A}\mathbf{c}\) preserves distance and correlation exactly because \(\mathbf{A}\) is orthonormal. The perturbed matrix is \(\mathbf{A}^{\mathrm{pert}}=\mathbf{A}+\mathbf{E}\), where \(\mathbf{E}=\operatorname{diag}(u_1\sigma,\ldots,u_K\sigma)\). Because \(\mathbf{A}^{\mathrm{pert}}\) is not orthonormal in general, the exact preservation guarantee does not automatically transfer. We therefore quantify the resulting deviation.

\begin{lemma}
\label{lem:perturbation-operator-bound}
Let \(\mathbf{E}=\operatorname{diag}(u_1\sigma,\ldots,u_K\sigma)\), where \(u_j\in[0,1]\) for every \(j\), and define \(\varepsilon(\sigma)=2\sigma+\sigma^2\). Then
\(\|(\mathbf{A}^{\mathrm{pert}})^{\top}\mathbf{A}^{\mathrm{pert}}-\mathbf{I}_K\|_{\mathrm{op}}\leq\varepsilon(\sigma)\),
where \(\|\cdot\|_{\mathrm{op}}\) denotes the spectral operator norm. The bound holds with probability \(1\).
\end{lemma}

\begin{proof}
Because \(\mathbf{A}\) is symmetric and orthonormal, \(\mathbf{A}^{\top}=\mathbf{A}\) and \(\mathbf{A}^{\top}\mathbf{A}=\mathbf{I}_K\). Moreover, \(\mathbf{E}\) is diagonal and therefore symmetric. Consequently,
\begin{equation*}
\begin{aligned}
(\mathbf{A}^{\mathrm{pert}})^{\top}\mathbf{A}^{\mathrm{pert}}
&=(\mathbf{A}+\mathbf{E})^{\top}(\mathbf{A}+\mathbf{E}) \\
&=\mathbf{A}^{\top}\mathbf{A}
  +\mathbf{A}\mathbf{E}
  +\mathbf{E}\mathbf{A}
  +\mathbf{E}^2 \\
&=\mathbf{I}_K
  +\mathbf{A}\mathbf{E}
  +\mathbf{E}\mathbf{A}
  +\mathbf{E}^2.
\end{aligned}
\end{equation*}
By the triangle inequality and submultiplicativity of the operator norm,
\(\|\mathbf{A}\mathbf{E}+\mathbf{E}\mathbf{A}+\mathbf{E}^2\|_{\mathrm{op}}
\leq 2\|\mathbf{A}\|_{\mathrm{op}}\|\mathbf{E}\|_{\mathrm{op}}
+\|\mathbf{E}\|_{\mathrm{op}}^2\).
Since \(\mathbf{A}\) has eigenvalues \(\pm1\) (Lemma 1, main paper), \(\|\mathbf{A}\|_{\mathrm{op}}=1\). Moreover, since \(u_j\in[0,1]\), \(\|\mathbf{E}\|_{\mathrm{op}}=\sigma\max_j u_j\leq\sigma\). Hence,
\(\|\mathbf{A}\mathbf{E}+\mathbf{E}\mathbf{A}+\mathbf{E}^2\|_{\mathrm{op}}
\leq 2\sigma+\sigma^2=\varepsilon(\sigma)\).  Because \(u_j\in[0,1]\) is a hard constraint imposed by the sampling procedure, rather than merely a high-probability event, the bound holds with probability \(1\), not merely in expectation.
\end{proof}

\begin{theorem}[Bounded Distance Distortion]
\label{thm:bounded-distance-distortion}
Let \(\varepsilon(\sigma)<1\). For any \(\mathbf{c},\mathbf{c}'\in\mathbb{R}_{>0}^K\) and any realization of \(\mathbf{A}^{\mathrm{pert}}\),
\(\sqrt{1-\varepsilon(\sigma)}\,\|\mathbf{c}-\mathbf{c}'\|
\leq
\|\mathbf{A}^{\mathrm{pert}}\mathbf{c}
-\mathbf{A}^{\mathrm{pert}}\mathbf{c}'\|
\leq
\sqrt{1+\varepsilon(\sigma)}\,\|\mathbf{c}-\mathbf{c}'\|\).
\end{theorem}

\begin{proof}
Let \(\mathbf{x}=\mathbf{c}-\mathbf{c}'\). By Lemma~\ref{lem:perturbation-operator-bound},
\(\left|
\mathbf{x}^{\top}
\bigl((\mathbf{A}^{\mathrm{pert}})^{\top}
\mathbf{A}^{\mathrm{pert}}-\mathbf{I}_K\bigr)
\mathbf{x}
\right|
\leq
\varepsilon(\sigma)\|\mathbf{x}\|^2\).
It follows that
\(\left|
\|\mathbf{A}^{\mathrm{pert}}\mathbf{x}\|^2-\|\mathbf{x}\|^2
\right|
\leq
\varepsilon(\sigma)\|\mathbf{x}\|^2\).
Rearranging and taking square roots yields the stated bounds.
\end{proof}

\begin{corollary}[Bounded Correlation Distortion]
\label{cor:bounded-correlation-distortion}
For any \(\mathbf{c},\mathbf{c}'\in\mathbb{R}_{>0}^K\), if \(\varepsilon(\sigma)<1\), then
\(\left|
\operatorname{corr}(\mathbf{A}^{\mathrm{pert}}\mathbf{c},
\mathbf{A}^{\mathrm{pert}}\mathbf{c}')
-\operatorname{corr}(\mathbf{c},\mathbf{c}')
\right|
\leq
2\varepsilon(\sigma)/(1-\varepsilon(\sigma))\).
\end{corollary}

\begin{proof}
Expanding
\(\langle
\mathbf{A}^{\mathrm{pert}}\mathbf{c},
\mathbf{A}^{\mathrm{pert}}\mathbf{c}'
\rangle
=
\langle\mathbf{c},\mathbf{c}'\rangle
+
\mathbf{c}^{\top}
(\mathbf{A}\mathbf{E}+\mathbf{E}\mathbf{A}+\mathbf{E}^2)
\mathbf{c}'\)
and applying the Cauchy--Schwarz inequality bounds the deviation of the numerator from \(\langle\mathbf{c},\mathbf{c}'\rangle\) by
\(\varepsilon(\sigma)\|\mathbf{c}\|\|\mathbf{c}'\|\).
Theorem~\ref{thm:bounded-distance-distortion} bounds the corresponding denominator relative to
\(\|\mathbf{c}\|\|\mathbf{c}'\|\)
within the multiplicative interval
\([1-\varepsilon(\sigma),1+\varepsilon(\sigma)]\).
Combining these bounds using the quotient rule gives the stated result.
\end{proof}

\begin{remark}
As \(\sigma\to0\), equivalently \(\rho\to\infty\), \(\varepsilon(\sigma)\to0\), and both distortion bounds vanish, recovering the exact guarantee of Lemma 1 of the main paper. Distortion increases monotonically and smoothly with \(\sigma\), providing a fully quantified trade-off rather than a discontinuous loss of the preservation property. We empirically verified that this bound is nearly saturated across the tested range of \(\sigma\).
\end{remark}

\subsection{Single-Query Magnitude Ambiguity}

The adversary is assumed to know \(\mathbf{A}\), \(K\), and \(\sigma\). The adversary also knows that \(\mathbf{c}\) is a valid confidence vector and therefore satisfies \(\sum_{j=1}^{K}c_j=1\). Combined with Lemma 2 of the main paper, observing \(\mathbf{p}\) reveals \(\operatorname{rank}(\mathbf{c})\) exactly.

\begin{theorem}[Reconstruction Interval]
\label{thm:reconstruction-interval}
For each class \(j\), let \(k_j\) denote its recovered rank and let its corresponding sampling interval be \(I_{k_j}=[L_j,U_j]\). Then
\(c_j\in[\underline{c}_j,\overline{c}_j]\), where
\(\underline{c}_j=(p_j+2/K)/(1+\sigma U_j)\) and
\(\overline{c}_j=(p_j+2/K)/(1+\sigma L_j)\).
\end{theorem}

\begin{proof}
From the closed form established in the proof of Lemma 2,
\(p_j=c_j(1+u_j\sigma)-2/K\), and therefore
\(c_j=(p_j+2/K)/(1+u_j\sigma)\).
Because \(u_j\in[L_j,U_j]\) and \(c_j\) is strictly decreasing as a function of \(u_j\), substituting the two interval endpoints gives the stated lower and upper bounds.
\end{proof}

\begin{theorem}[Non-Degenerate Ambiguity]
\label{thm:nondegenerate-ambiguity}
Consider any two classes \(a\neq b\) whose draws \(u_a\) and \(u_b\) lie in the interiors of their respective sampling intervals, which occurs with probability \(1\). Then there exists \(\delta_0>0\) such that, for every \(\delta\in(-\delta_0,\delta_0)\), the vector \(\mathbf{c}'\) defined by
\(c'_a=c_a+\delta\), \(c'_b=c_b-\delta\), and \(c'_j=c_j\) for \(j\notin\{a,b\}\)
is a valid confidence vector that produces the same observed vector \(\mathbf{p}\). In particular, one may choose
\(\delta_0=\min\{
\overline{c}_a-c_a,\,
c_a-\underline{c}_a,\,
c_b-\underline{c}_b,\,
\overline{c}_b-c_b
\}\),
where the bounds are those given in Theorem~\ref{thm:reconstruction-interval}.
\end{theorem}

\begin{proof}
For class \(a\), preserving the observed value \(p_a\) requires
\(u'_a=[(p_a+2/K)/c'_a-1]/\sigma\in[L_a,U_a]\).
Solving this condition for \(\delta\) produces an interval containing \(0\) with positive radius because \(u_a\) lies in the interior of \([L_a,U_a]\). The same argument for class \(b\), using \(c'_b=c_b-\delta\), for class \(b\), using \(c'_b=c_b-\delta\), produces a second interval containing \(0\) with positive radius. Their intersection therefore contains a nonempty open interval around \(0\). The perturbation preserves \(\sum_j c'_j=1\), and choosing \(|\delta|<\delta_0\) preserves feasibility of both modified coordinates.
\end{proof}

\begin{remark}
Theorem~\ref{thm:nondegenerate-ambiguity} establishes that \(\mathbf{p}\) is consistent with an uncountable continuum of distinct valid confidence vectors, rather than merely providing an interval bound that happens to contain the true \(\mathbf{c}\). The interval width in Theorem~\ref{thm:reconstruction-interval} scales with \(\sigma=C/\rho\). Consequently, smaller values of \(\rho\), and hence larger values of \(\sigma\), directly increase single-query ambiguity.
\end{remark}

\subsection{Exact Characterization Under Repeated Queries}
In \texttt{PRIVEE-U}, Because \(u_j\) is drawn from a fully public and bounded distribution, its extreme values become recoverable in the limit. We characterize this behavior precisely because an exact convergence rate enables a principled and provable query-budget mitigation.

For a class \(j\) with known interval floor \(L_j\), define the estimator after \(T\) independent queries on the same confidence vector by
\(\widehat{c}_j^{(T)}
=
\bigl(\min_{1\leq t\leq T}p_j^{(t)}+2/K\bigr)/(1+\sigma L_j)\).

\begin{theorem}[Exact Convergence Rate]
\label{thm:exact-convergence-rate}
Let
\(\kappa_j=K(1+\sigma L_j)/(c_j\sigma)\).
For every \(\delta'\in[0,1/\kappa_j]\),
\begin{equation}
\Pr\!\left[
\widehat{c}_j^{(T)}-c_j<\delta'
\right]
=
1-\left(1-\kappa_j\delta'\right)^T.
\label{eq:exact-convergence-rate}
\end{equation}
\end{theorem}

\begin{proof}
From the closed form in Lemma~0.2,
\(p_j^{(t)}=c_j(1+u_j^{(t)}\sigma)-2/K\).
Let \(m_T=\min_{1\leq t\leq T}u_j^{(t)}\). Substitution into the estimator gives
\(\widehat{c}_j^{(T)}
=
c_j(1+\sigma m_T)/(1+\sigma L_j)\),
and hence
\(\widehat{c}_j^{(T)}-c_j
=
c_j\sigma(m_T-L_j)/(1+\sigma L_j)\).
This quantity is strictly increasing in \(m_T\).

Each \(u_j^{(t)}\) is independently distributed as
\(\operatorname{Uniform}(L_j,U_j)\), where \(U_j-L_j=1/K\). Therefore,
\(\Pr[u_j^{(t)}<L_j+x]=Kx\) for \(x\in[0,1/K]\), and
\(\Pr[m_T<L_j+x]=1-(1-Kx)^T\).
Setting
\(x=\delta'(1+\sigma L_j)/(c_j\sigma)\)
gives \(Kx=\kappa_j\delta'\). Substitution yields
Equation~\eqref{eq:exact-convergence-rate}.
\end{proof}

We verified this closed form against simulation (Section \ref{sec:repeated-query-experiment}) for multiple values of \(T\) and \(\delta'\). For example, when \(T=1000\), the predicted and empirical probabilities agreed within simulation noise at every tested tolerance. 

\begin{corollary}[Query-Budget Guarantee]
\label{cor:query-budget-guarantee}
To guarantee
\(\Pr[\widehat{c}_j^{(T)}-c_j<\delta']\leq p_{\max}\)
for a target tolerance \(\delta'\), the number of repeated queries permitted for the same input must satisfy
\(T\leq
\log(1-p_{\max})/
\log(1-\kappa_j\delta')\).
For an integer-valued query budget, the right-hand side may be replaced by its floor.
\end{corollary}

\begin{remark}
The same coordinate-wise argument applies across all \(K\) classes using the same batch of repeated queries because each class-specific draw \(u_j\) is sampled independently during every query. Empirically, full-vector recovery converges at approximately the rate of the hardest coordinate. A structurally similar but looser bound can be derived for a sample-mean estimator by replacing the interval endpoint \(L_j\) with \(\mathbb{E}[u_j]\). The minimum-based estimator converges at rate \(O(1/T)\), whereas the sample-mean estimator converges at rate \(O(1/\sqrt{T})\). The minimum-based estimator is therefore the binding case for Corollary~\ref{cor:query-budget-guarantee}.
\end{remark}

\subsection{Limitations and Practical Considerations}
\label{sec:limitations-practical}

\texttt{PRIVEE-U} preserves class ranking and bounds distance/correlation distortion (Lemma~\ref{lem:perturbation-operator-bound}, Theorem~\ref{thm:bounded-distance-distortion}), but its guarantee is an ambiguity/reconstruction-hardness bound, not a formal $(\varepsilon,\delta)$-differential-privacy guarantee: the perturbation $u_j\sigma$ is drawn from a bounded, rank-dependent (non-symmetric) interval rather than a calibrated DP noise distribution, so the standard DP composition and post-processing theorems do not apply directly. \texttt{PRIVEE-U} should therefore be treated as a pragmatic, inference-time defense against feature-inference attacks rather than a substitute for DP in settings that require a provable guarantee.

Regarding repeated queries, Theorem~\ref{thm:exact-convergence-rate} and Corollary~\ref{cor:query-budget-guarantee} give an exact convergence rate and a derived query-budget bound: an adversary who queries the same input $T$ times can drive a minimum-based estimator's error below any target tolerance $\delta'$ with probability approaching $1$ as $T$ grows, at a rate governed by $\kappa_j=K(1+\sigma L_j)/(c_j\sigma)$. This means protection degrades under sustained querying of a fixed input, and a deployment should impose a query budget consistent with Corollary~\ref{cor:query-budget-guarantee} for its chosen $\rho$. 
While the main text notes that empirical validation of this bound remains for future work, we take a first step here by evaluating it against simulated repeated-query attacks on fixed MNIST and CIFAR-10 inputs (Section~\ref{sec:repeated-query-experiment}); a full empirical validation across additional datasets, and extending the analysis to \texttt{PRIVEE-U+}, remain for future work.

\texttt{PRIVEE-U} also assumes an honest, non-colluding coordinator that applies the perturbation faithfully before release (main paper, Section~3); a compromised or colluding coordinator can simply release unperturbed scores, in which case \texttt{PRIVEE} provides no protection. This is a deployment-trust assumption shared with most inference-time defenses that rely on a mediating party, and is not addressed by the theoretical results above.

\texttt{PRIVEE-U}'s shared perturbation scale $\sigma$ also becomes less effective as $K$ grows because the $K$ equal-width sub-intervals from which each $u_j$ is drawn narrow proportionally to $1/K$ (Algorithm 1, main paper), leaving less room to perturb individual classes. This is evident in the CIFAR-100 ($K=100$) results under high attack strength (Table~\ref{tab:grna75}; see also the main paper's Results \& Analyses). \texttt{PRIVEE-U+}'s per-class scales $\sigma_j$ avoid this limitation and are recommended for large-$K$ deployments.

\begin{table*}[t]
\centering
\footnotesize
\setlength{\tabcolsep}{5pt}
\renewcommand{\arraystretch}{1.15}
\caption{MSE of the Gradient Inversion Attack (GIA) under different defense mechanisms with attack strength $50\%$. Higher MSE indicates stronger resistance to gradient inversion.}
\label{tab:mse_50}

\begin{tabular}{lrrrrrrrr}
\hline
\textbf{Dataset} &
\textbf{No Defense} &
\textbf{R(1)} &
\textbf{R(2)} &
\textbf{OPE} &
\textbf{DP ($\varepsilon=0.5$)} &
\textbf{DP ($\varepsilon=0.7$)} &
\textbf{\texttt{PRIVEE-U}} &
\textbf{\texttt{PRIVEE-U+}} \\
\hline
MNIST
& 0.0340 & 0.1158 & 0.1127 & 0.3342 & 0.6363 & 0.4811 & 22.1090 & \textbf{27.7500} \\

CIFAR-100
& 0.2600 & 0.1241 & 0.1270 & 0.3580 & 0.5469 & 0.4224 & \textbf{21.9023} & 21.3105 \\

CIFAR-10
& 0.2400 & 0.1255 & 0.1235 & 0.3911 & 0.6510 & 0.5531 & \textbf{22.5932} & 18.2900 \\
\hline
\end{tabular}
\end{table*}

\begin{table*}[t]
\centering
\footnotesize
\setlength{\tabcolsep}{5pt}
\renewcommand{\arraystretch}{1.15}
\caption{MSE of the Gradient Inversion Attack (GIA) under different defense mechanisms with attack strength $75\%$. Higher MSE indicates stronger resistance to gradient inversion.}
\label{tab:mse_75}

\begin{tabular}{lrrrrrrrr}
\hline
\textbf{Dataset} &
\textbf{No Defense} &
\textbf{R(1)} &
\textbf{R(2)} &
\textbf{OPE} &
\textbf{DP ($\varepsilon=0.5$)} &
\textbf{DP ($\varepsilon=0.7$)} &
\textbf{\texttt{PRIVEE-U}} &
\textbf{\texttt{PRIVEE-U+}} \\
\hline
MNIST
& 0.1839 & 0.1567 & 0.1584 & 0.3555 & 0.6891 & 0.5408 & 24.7905 & \textbf{24.9554} \\

CIFAR-100
& 0.4210 & 0.1343 & 0.1150 & 0.3109 & 0.4950 & 0.4141 & \textbf{23.6234} & 20.9240 \\

CIFAR-10
& 0.3423 & 0.1330 & 0.1431 & 0.3722 & 0.7323 & 0.5966 & \textbf{41.8012} & 19.4980 \\
\hline
\end{tabular}
\end{table*}

\begin{table*}[t]
\centering
\footnotesize
\setlength{\tabcolsep}{5pt}
\renewcommand{\arraystretch}{1.15}
\caption{MSE of the Gradient Inversion Attack (GIA) under different defense mechanisms with attack strength $25\%$. Higher MSE indicates stronger resistance to gradient inversion.}
\label{tab:mse_25}

\begin{tabular}{lrrrrrrrr}
\hline
\textbf{Dataset} &
\textbf{No Defense} &
\textbf{R(1)} &
\textbf{R(2)} &
\textbf{OPE} &
\textbf{DP ($\varepsilon=0.5$)} &
\textbf{DP ($\varepsilon=0.7$)} &
\textbf{\texttt{PRIVEE-U}} &
\textbf{\texttt{PRIVEE-U+}} \\
\hline
MNIST
& 0.0180 & 0.0681 & 0.0649 & 0.2988 & 0.5891 & 0.4531 & 20.5962 & \textbf{26.0250} \\

CIFAR-100
& 0.1400 & 0.1189 & 0.1026 & 0.3202 & 0.5407 & 0.4666 & \textbf{23.3077} & 21.0500 \\

CIFAR-10
& 0.1350 & 0.1177 & 0.0960 & 0.4119 & 0.6100 & 0.5158 & \textbf{40.2123} & 17.8243 \\
\hline
\end{tabular}
\end{table*}

\subsubsection{Computational Complexity Analysis}

The computational cost of \texttt{PRIVEE-U} and \texttt{PRIVEE-U+} is efficient, scaling linearly with the number of classes $K$. Specifically:
\begin{itemize}
    \item Applying the transformation matrix $\mathbf{A}$ to the confidence vector $\mathbf{c}$ has a time complexity of $O(K)$, owing to the structured or diagonal nature of $\mathbf{A}$.
    \item Incorporating the rank-aware diagonal perturbations (i.e., computing $\mathbf{A}^{\text{pert}}$) also requires $O(K)$ time, as only diagonal entries are modified.
    \item Notably, matrix $\mathbf{A}$ is fixed and constructed once, allowing reuse across all confidence vectors without additional cost.
\end{itemize}

Thus, the overall per-vector complexity is $O(K)$, making \texttt{PRIVEE-U} and \texttt{PRIVEE-U+} suitable for real-time and large-scale inference.

\section{Evaluated Feature-Inference Attacks}

We evaluate \texttt{PRIVEE} against the feature-inference attacks of \citet{jiang2022comprehensive}. Table~\ref{tab:evaluated-attacks} summarizes their objectives and mechanisms. We exclude the Equation-Solving attack because it requires access to the original (unperturbed) confidence scores, whereas \texttt{PRIVEE} releases only perturbed scores during inference.
\begin{table}[h]
\centering
\footnotesize
\setlength{\tabcolsep}{3pt}
\renewcommand{\arraystretch}{1.1}
\caption{Feature-inference attacks evaluated in this work.}
\label{tab:evaluated-attacks}
\begin{tabular}{p{0.20\columnwidth}p{0.70\columnwidth}}
\hline
\textbf{Attack} & \textbf{Description} \\
\hline
\textbf{GRNA} &
Uses a conditional generative regression network to reconstruct passive features from active features and observed confidence scores. \\

\textbf{GIA} &
Optimizes reconstructed features by matching the confidence scores produced by the VFL model. \\

\textbf{Equation-Solving} &
Recovers unknown features from log-probability equations using unperturbed confidence scores. Excluded because \texttt{PRIVEE} never releases unperturbed confidence values. \\
\hline
\end{tabular}
\end{table}

\section{Additional Experimental Results}

\subsection{GRN Attack (75\% Attack Strength)}

Table~\ref{tab:grna75} reports the MSE of the GRN attack across datasets at 75\% attack strength, complementing the 25\% and 50\% results reported in the main paper (Tables 4 and 5). The pattern is broadly consistent with the lower attack strengths: \texttt{PRIVEE-U} and \texttt{PRIVEE-U+} substantially outperform all other defenses on every dataset, including CIFAR-100, where \texttt{PRIVEE-U} ($1.8043$) still exceeds both DP variants ($0.9273$ and $0.4761$). CIFAR-100 is nonetheless the one dataset where \texttt{PRIVEE-U+} ($3.1780$) is needed to substantially widen the margin over \texttt{PRIVEE-U}: at $K=100$, \texttt{PRIVEE-U}'s MSE is an order of magnitude lower than its own MSE on every other dataset at this attack strength (5.84--18.28), consistent with the large-$K$ discussion in the main paper's Results \& Analyses section.

\begin{table*}[t]
\centering
\footnotesize
\setlength{\tabcolsep}{5pt}
\renewcommand{\arraystretch}{1.15}
\caption{MSE of the GRN attack under different defense mechanisms across datasets with attack strength $75\%$. Higher MSE indicates stronger resistance to gradient reconstruction attacks.}
\label{tab:grna75}

\begin{tabular}{lrrrrrrrr}
\hline
\textbf{Dataset} &
\textbf{No Defense} &
\textbf{R(1)} &
\textbf{R(2)} &
\textbf{OPE} &
\textbf{DP ($\varepsilon=0.5$)} &
\textbf{DP ($\varepsilon=0.7$)} &
\textbf{\texttt{PRIVEE-U}} &
\textbf{\texttt{PRIVEE-U+}} \\
\hline
MNIST
& 0.1434 & 0.1438 & 0.1435 & 0.3404 & 0.9441 & 0.5970 & \textbf{17.5048} & 2.6346 \\

CIFAR100
& 0.3700 & 0.0129 & 0.0127 & 0.2894 & 0.9273 & 0.4761 & 1.8043 & \textbf{3.1780} \\

CIFAR10
& 0.3100 & 0.1291 & 0.1213 & 0.3586 & 1.0583 & 0.5744 & \textbf{7.0418} & 3.6539 \\

Drive Diagnosis
& 0.2219 & 0.1233 & 0.1212 & 0.3356 & 1.0551 & 0.5770 & \textbf{18.2788} & 2.4318 \\

Adult Income
& 0.3511 & 0.3572 & 0.3045 & 0.2936 & 1.1323 & 0.7071 & \textbf{5.8399} & 3.8318 \\
\hline
\end{tabular}
\end{table*}

\subsection{\bf Gradient Inversion Attack}
Tables~\ref{tab:mse_50}, \ref{tab:mse_75}, and \ref{tab:mse_25} report the mean squared error (MSE) between reconstructed and original confidence scores under Gradient Inversion Attack (GIA) across datasets and attack strengths. As with GRN, the \textit{No Defense} setting consistently yields very low MSEs, indicating near-perfect recovery of passive-party data.

Applying simple rounding increases reconstruction error by roughly 3–4$\times$, while OPE leads to even higher errors in the 0.30–0.40 range. Adding DP noise further raises the MSE into the 0.42–0.73 range. However, only \texttt{PRIVEE-U} and \texttt{PRIVEE-U+} elevate the error into the tens. For example, at 50\% attack strength, MSE on MNIST reaches 22.11 with \texttt{PRIVEE-U} and 27.75 with \texttt{PRIVEE-U+}; on CIFAR-10, the scores are 22.59 and 18.29, respectively.

This trend holds across other attack strengths: even at 25\%, \texttt{PRIVEE-U} maintains MSE above 20, while all other defenses stay below 1. These results show that while rounding and OPE offer limited protection and standard DP adds moderate noise, only \texttt{PRIVEE-U} and \texttt{PRIVEE-U+} effectively defend against gradient inversion, increasing reconstruction error by up to two orders of magnitude.
\begin{table*}[t]
\centering
\footnotesize
\setlength{\tabcolsep}{7pt}
\renewcommand{\arraystretch}{1.1}
\caption{Baseline training and inference accuracy of the VFL models before applying attacks or defenses.}
\label{tab:accuracy_baseline}

\begin{tabular}{lcrrr}
\hline
\textbf{Dataset} &
\shortstack{\textbf{\#}\\\textbf{Classes}} &
\textbf{Model} &
\shortstack{\textbf{Train}\\\textbf{Acc. (\%)}} &
\shortstack{\textbf{Infer.}\\\textbf{Acc. (\%)}} \\
\hline
MNIST            & 10  & NN      & 98.34 & 97.64 \\
CIFAR-100        & 100 & ResNet  & 71.24 & 58.98 \\
CIFAR-10         & 10  & ResNet  & 75.59 & 63.21 \\
Drive Diagnosis  & 11  & NN, LR  & 79.06 & 80.90 \\
Adult Income     & 2   & NN, LR  & 74.65 & 74.27 \\
\hline
\end{tabular}
\end{table*}
\subsection{Training Accuracy of VFL Models}
Table~\ref{tab:accuracy_baseline} presents the training and inference accuracies achieved by our 2-party VFL model on each dataset. The baseline model consistently attained high accuracy across all datasets before the introduction of any adversarial attacks or defense mechanisms.

\subsection{Empirical Validation of the Repeated-Query Bound}
\label{sec:repeated-query-experiment}

The following analysis shows that
PRIVEE-U does not provide asymptotic protection when an adversary
can submit the same input repeatedly. For a fixed input, the original
confidence vector
\begin{equation*}
\mathbf{c}=(c_1,\ldots,c_K)^\top
\end{equation*}
remains unchanged across queries, whereas PRIVEE-U independently
resamples the rank-dependent perturbation variables
\begin{equation*}
u_j^{(t)} \sim \operatorname{Uniform}(L_j,U_j),
\qquad
U_j-L_j=\frac{1}{K},
\end{equation*}
for every coordinate \(j\) and query \(t\).

Because PRIVEE-U preserves the class ranking, the adversary can
identify the rank-dependent interval \([L_j,U_j]\) associated with
each confidence coordinate. The adversary is also assumed to know
the public defense parameters \(K\), \(A\), and
\(\sigma=C/\rho\). For
\begin{equation*}
A=I_K-\frac{2}{K}\mathbf{1}_K\mathbf{1}_K^\top
\end{equation*}
and a valid confidence vector satisfying
\(\sum_{j=1}^{K}c_j=1\), the released value for coordinate \(j\)
during query \(t\) is
\begin{equation*}
p_j^{(t)}
=
c_j\left(1+\sigma u_j^{(t)}\right)-\frac{2}{K}.
\end{equation*}

Since \(p_j^{(t)}\) is increasing in \(u_j^{(t)}\), the smallest
released value observed after \(T\) repeated queries corresponds to
the smallest sampled perturbation. Define
\begin{equation*}
m_{j,T}
=
\min_{1\leq t\leq T}u_j^{(t)}.
\end{equation*}
The adversary then applies the minimum-based estimator
\begin{equation*}
\widehat{c}_j^{(T)}
=
\frac{
\min_{1\leq t\leq T}p_j^{(t)}+2/K
}{
1+\sigma L_j
}.
\end{equation*}

Substituting the expression for the minimum released value gives
\begin{equation*}
\widehat{c}_j^{(T)}-c_j
=
\frac{
c_j\sigma
}{
1+\sigma L_j
}
\left(m_{j,T}-L_j\right).
\end{equation*}
Thus, the reconstruction error is determined exactly by how close
the smallest sampled perturbation is to the known lower endpoint
\(L_j\). Since the minimum of \(T\) independent uniform samples
satisfies
\begin{equation*}
\mathbb{E}[m_{j,T}-L_j]
=
\frac{1}{K(T+1)},
\end{equation*}
the expected absolute reconstruction error for coordinate \(j\) is
\begin{equation*}
\mathbb{E}
\left[
\left|
\widehat{c}_j^{(T)}-c_j
\right|
\right]
=
\frac{
c_j\sigma
}{
K(1+\sigma L_j)(T+1)
}.
\end{equation*}

For one fixed confidence vector, we measure reconstruction using the
mean absolute per-coordinate error
\begin{equation*}
\operatorname{MAE}^{(T)}(\mathbf{c})
=
\frac{1}{K}
\sum_{j=1}^{K}
\left|
\widehat{c}_j^{(T)}-c_j
\right|.
\end{equation*}
Its theoretical expectation is
\begin{equation*}
\mathbb{E}
\left[
\operatorname{MAE}^{(T)}(\mathbf{c})
\right]
=
\frac{\sigma}{K^2(T+1)}
\sum_{j=1}^{K}
\frac{c_j}{1+\sigma L_j}.
\end{equation*}
The expected error therefore decreases at rate
\(\mathcal{O}(1/T)\).

\paragraph{Experiment Details.}
For each dataset, we select one fixed test input and the same confidence vector is
then protected repeatedly using fresh, independent PRIVEE-U
perturbations. We evaluate
\begin{equation*}
T\in\{1,2,5,10,20,50,100,200,500,1000\}
\end{equation*}
and
\begin{equation*}
\rho\in\{0.1,0.3,0.5,0.7\},
\qquad
\sigma=\frac{0.485}{\rho}.
\end{equation*}
For each combination of \(T\) and \(\rho\), the repeated-query
simulation is independently replicated 20 times. The empirical
curve reports the mean value of
\(\operatorname{MAE}^{(T)}(\mathbf{c})\) across these independent
replications, and the shaded region reports the corresponding
\(95\%\) confidence interval. The theoretical curve reports
\(\mathbb{E}[\operatorname{MAE}^{(T)}(\mathbf{c})]\) for the same
fixed confidence vector.

\subsubsection{Results}

Figure~\ref{fig:repeated-query-mnist-cifar} compares the empirical
and theoretical reconstruction errors for an MNIST confidence
vector and a CIFAR-10 confidence vector. In both cases, the
empirical error closely follows the exact theoretical expectation and
decreases approximately linearly with slope \(-1\) on the log-log
scale, consistent with the predicted \(\mathcal{O}(1/T)\) convergence
rate. Smaller values of \(\rho\) produce larger reconstruction errors
at every finite query count because they correspond to larger
perturbation scales \(\sigma=C/\rho\). Nevertheless, the error
approaches zero for every evaluated value of \(\rho\) as the number
of repeated queries increases.

The empirical curves do not coincide perfectly with the theoretical
curves because the empirical values are obtained from a finite number
of independent perturbation replications, whereas the theoretical
curves represent exact expectations over the perturbation
distribution. The observed agreement supports the repeated-query
analysis for the evaluated fixed inputs. These results also confirm
that PRIVEE-U should be deployed with repeated-input detection and a
query budget, rather than allowing unlimited repeated evaluations of
the same input.

\begin{figure*}[t]
\centering

\begin{minipage}[t]{0.49\textwidth}
    \centering
    \includegraphics[width=\linewidth]{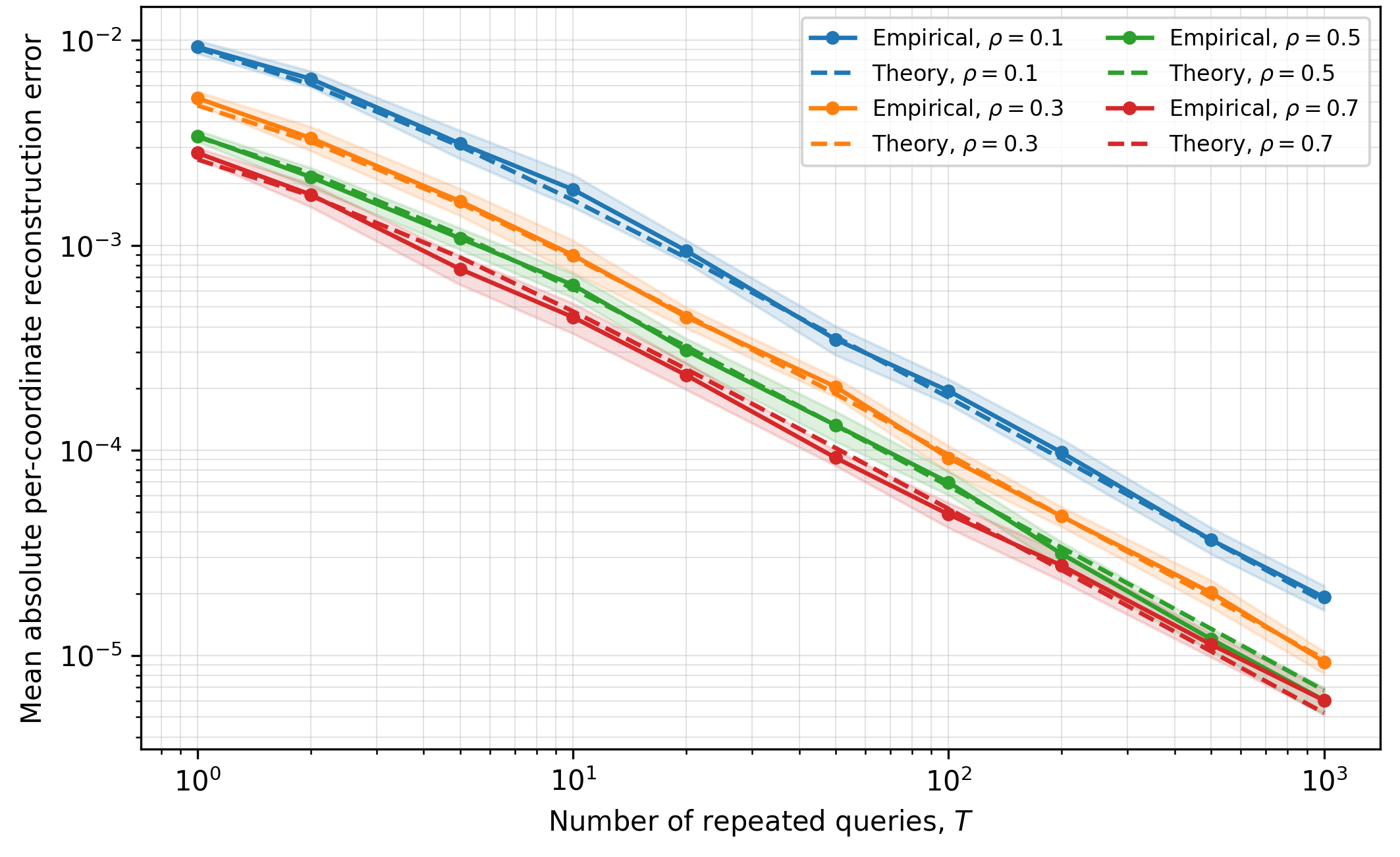}
\end{minipage}\hfill
\begin{minipage}[t]{0.49\textwidth}
    \centering
    \includegraphics[width=\linewidth]{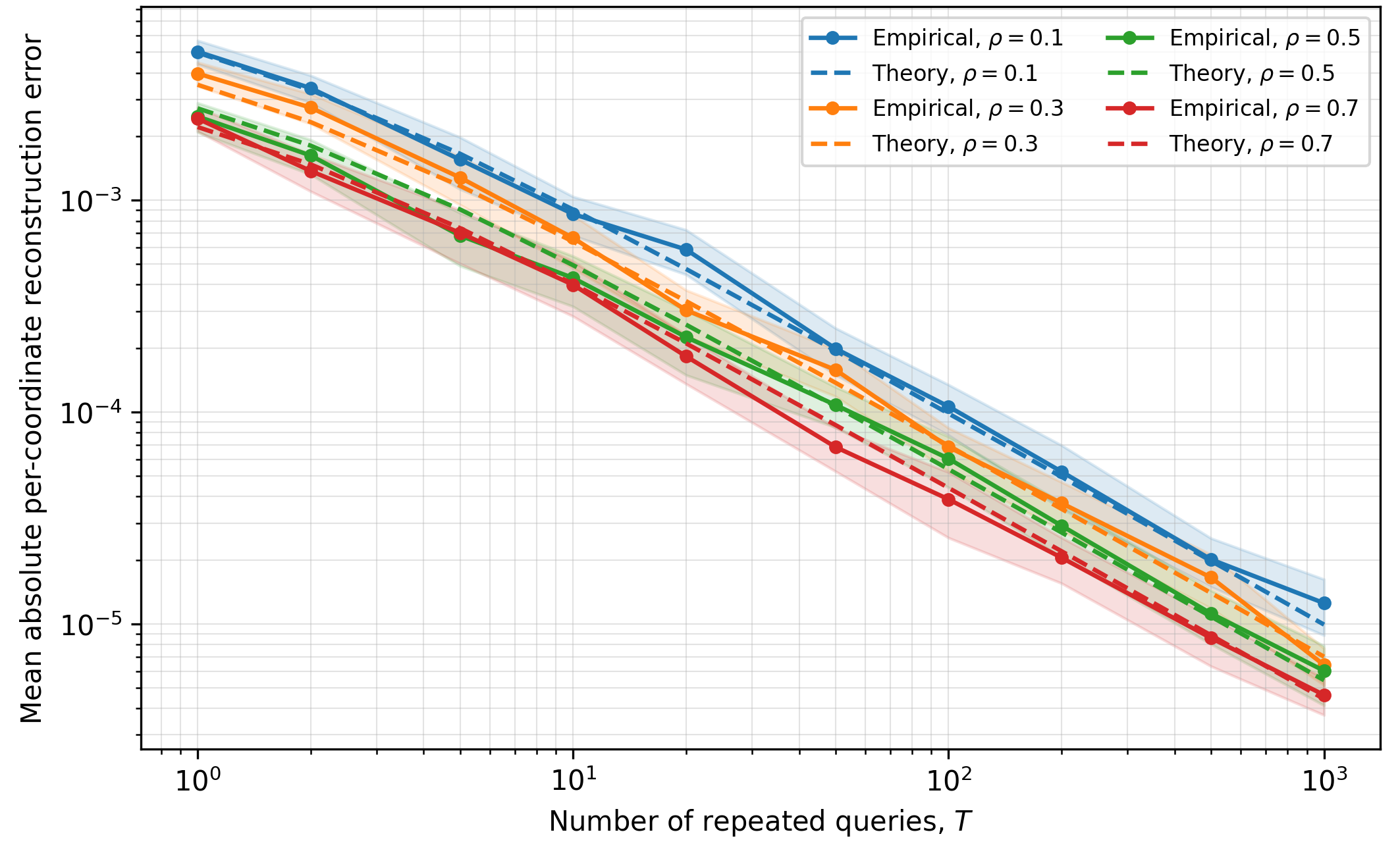}
\end{minipage}

\caption{
Single-input repeated-query reconstruction under PRIVEE-U on MNIST (left) and CIFAR-10 (right). For each dataset, one fixed confidence vector is perturbed repeatedly using fresh independent rank-dependent draws. The solid curves show the empirical mean absolute per-coordinate reconstruction error across independent repeated-query simulations, and the shaded regions show \(95\%\) confidence intervals. The dashed curves show the exact theoretical expectation for the same fixed confidence vector. In both datasets, the reconstruction error decreases approximately as \(\mathcal{O}(1/T)\), consistent with Theorem~\ref{thm:exact-convergence-rate}. Smaller values of \(\rho\), corresponding to larger perturbation scales \(\sigma=C/\rho\), produce larger finite-query reconstruction errors but do not prevent convergence under sustained repeated querying.
}
\label{fig:repeated-query-mnist-cifar}
\end{figure*}

\subsubsection{Numerical Example for the Repeated Query Analysis}

Consider a two-class confidence vector
\begin{equation*}
\mathbf{c} =
\begin{bmatrix}
0.2 \\
0.8
\end{bmatrix}.
\end{equation*}

The scores are already sorted from smallest to largest. Let
\begin{equation*}
K=2,
\qquad
\sigma=1,
\qquad
T=3.
\end{equation*}

The rank-dependent perturbation intervals are
\begin{equation*}
L =
\begin{bmatrix}
0 \\
0.5
\end{bmatrix},
\qquad
U =
\begin{bmatrix}
0.5 \\
1
\end{bmatrix}.
\end{equation*}

Thus, the random perturbation variables satisfy
\begin{equation*}
u_1^{(t)} \sim \operatorname{Uniform}(0,0.5),
\qquad
u_2^{(t)} \sim \operatorname{Uniform}(0.5,1).
\end{equation*}

Suppose that the following random values are generated during
three repeated queries:
\begin{equation*}
\begin{array}{c|cc}
\text{Query } t & u_1^{(t)} & u_2^{(t)} \\
\hline
1 & 0.40 & 0.90 \\
2 & 0.10 & 0.70 \\
3 & 0.30 & 0.80
\end{array}
\end{equation*}

The minimum values observed across the three queries are
\begin{equation*}
m_{1,3}
=
\min_{1\leq t\leq 3}u_1^{(t)}
=
0.10,
\end{equation*}
and
\begin{equation*}
m_{2,3}
=
\min_{1\leq t\leq 3}u_2^{(t)}
=
0.70.
\end{equation*}

These minimum values are not equal to the corresponding lower
interval endpoints:
\begin{equation*}
L_1=0,
\qquad
L_2=0.5.
\end{equation*}

\paragraph{Empirical Reconstruction Error.}

For each coordinate \(j\), the perturbed score released during
query \(t\) is
\begin{equation*}
p_j^{(t)}
=
c_j\left(1+\sigma u_j^{(t)}\right)
-
\frac{2}{K}.
\end{equation*}

Because \(K=2\),
\begin{equation*}
\frac{2}{K}=1.
\end{equation*}

For the first coordinate, \(c_1=0.2\). The three perturbed
releases are
\begin{equation*}
p_1^{(1)}
=
0.2(1+0.40)-1
=
-0.72,
\end{equation*}
\begin{equation*}
p_1^{(2)}
=
0.2(1+0.10)-1
=
-0.78,
\end{equation*}
and
\begin{equation*}
p_1^{(3)}
=
0.2(1+0.30)-1
=
-0.74.
\end{equation*}

Therefore,
\begin{equation*}
\min_{1\leq t\leq 3}p_1^{(t)}
=
-0.78.
\end{equation*}

Using the known lower endpoint \(L_1=0\), the attacker estimates
the original confidence score as
\begin{equation*}
\widehat{c}_1^{(3)}
=
\frac{
\min_{1\leq t\leq 3}p_1^{(t)}+2/K
}{
1+\sigma L_1
}.
\end{equation*}

Substituting the numerical values gives
\begin{equation*}
\widehat{c}_1^{(3)}
=
\frac{-0.78+1}{1+1(0)}
=
0.22.
\end{equation*}

The corresponding absolute reconstruction error is
\begin{equation*}
\left|
\widehat{c}_1^{(3)}-c_1
\right|
=
|0.22-0.20|
=
0.02.
\end{equation*}

For the second coordinate, \(c_2=0.8\). The three perturbed
releases are
\begin{equation*}
p_2^{(1)}
=
0.8(1+0.90)-1
=
0.52,
\end{equation*}
\begin{equation*}
p_2^{(2)}
=
0.8(1+0.70)-1
=
0.36,
\end{equation*}
and
\begin{equation*}
p_2^{(3)}
=
0.8(1+0.80)-1
=
0.44.
\end{equation*}

Therefore,
\begin{equation*}
\min_{1\leq t\leq 3}p_2^{(t)}
=
0.36.
\end{equation*}

Using \(L_2=0.5\), the reconstructed confidence score is
\begin{equation*}
\widehat{c}_2^{(3)}
=
\frac{0.36+1}{1+1(0.5)}
=
\frac{1.36}{1.5}
\approx
0.9067.
\end{equation*}

The corresponding absolute reconstruction error is
\begin{equation*}
\left|
\widehat{c}_2^{(3)}-c_2
\right|
=
|0.9067-0.8|
=
0.1067.
\end{equation*}

The empirical mean absolute per-coordinate reconstruction error is
therefore
\begin{equation*}
\frac{
\left|\widehat{c}_1^{(3)}-c_1\right|
+
\left|\widehat{c}_2^{(3)}-c_2\right|
}{2}
=
\frac{0.02+0.1067}{2}
\approx
0.0633.
\end{equation*}

\paragraph{Theoretical Expected Reconstruction Error.}

The theoretical expected absolute reconstruction error for
coordinate \(j\) after \(T\) repeated queries is
\begin{equation*}
\mathbb{E}
\left[
\left|
\widehat{c}_j^{(T)}-c_j
\right|
\right]
=
\frac{
c_j\sigma
}{
K(1+\sigma L_j)(T+1)
}.
\end{equation*}

For the first coordinate,
\begin{equation*}
c_1=0.2,
\qquad
L_1=0.
\end{equation*}

Thus,
\begin{equation*}
\mathbb{E}
\left[
\left|
\widehat{c}_1^{(3)}-c_1
\right|
\right]
=
\frac{
0.2(1)
}{
2(1+1(0))(3+1)
}
=
\frac{0.2}{8}
=
0.025.
\end{equation*}

For the second coordinate,
\begin{equation*}
c_2=0.8,
\qquad
L_2=0.5.
\end{equation*}

Thus,
\begin{equation*}
\mathbb{E}
\left[
\left|
\widehat{c}_2^{(3)}-c_2
\right|
\right]
=
\frac{
0.8(1)
}{
2(1+1(0.5))(3+1)
}
=
\frac{0.8}{12}
\approx
0.0667.
\end{equation*}

The theoretical expected mean absolute per-coordinate error is
therefore
\begin{equation*}
\frac{0.025+0.0667}{2}
\approx
0.0458.
\end{equation*}

\subsection{Client-Count and $\rho$ Ablation for Other Datasets}
\label{sec:other-datasets-ablation}

The main text notes that generalizing the ``unaffected by client
count'' finding beyond the tested MNIST/GRNA setting remains to be
validated. As a first step in that direction, we examine here
how PRIVEE-U behaves on two additional datasets: CIFAR-10 as the
perturbation-control parameter \(\rho\) varies, and DRIVE as the number
of clients varies. Figure~\ref{fig:other-datasets-ablation} reports the
corresponding results. In both cases, higher MSE indicates
stronger resistance to feature reconstruction. A full replication
across both dimensions on additional datasets and attacks remains for
future work.

\begin{figure*}[t]
    \centering
    \begin{minipage}[t]{0.45\textwidth}
        \centering
        \includegraphics[width=\linewidth]{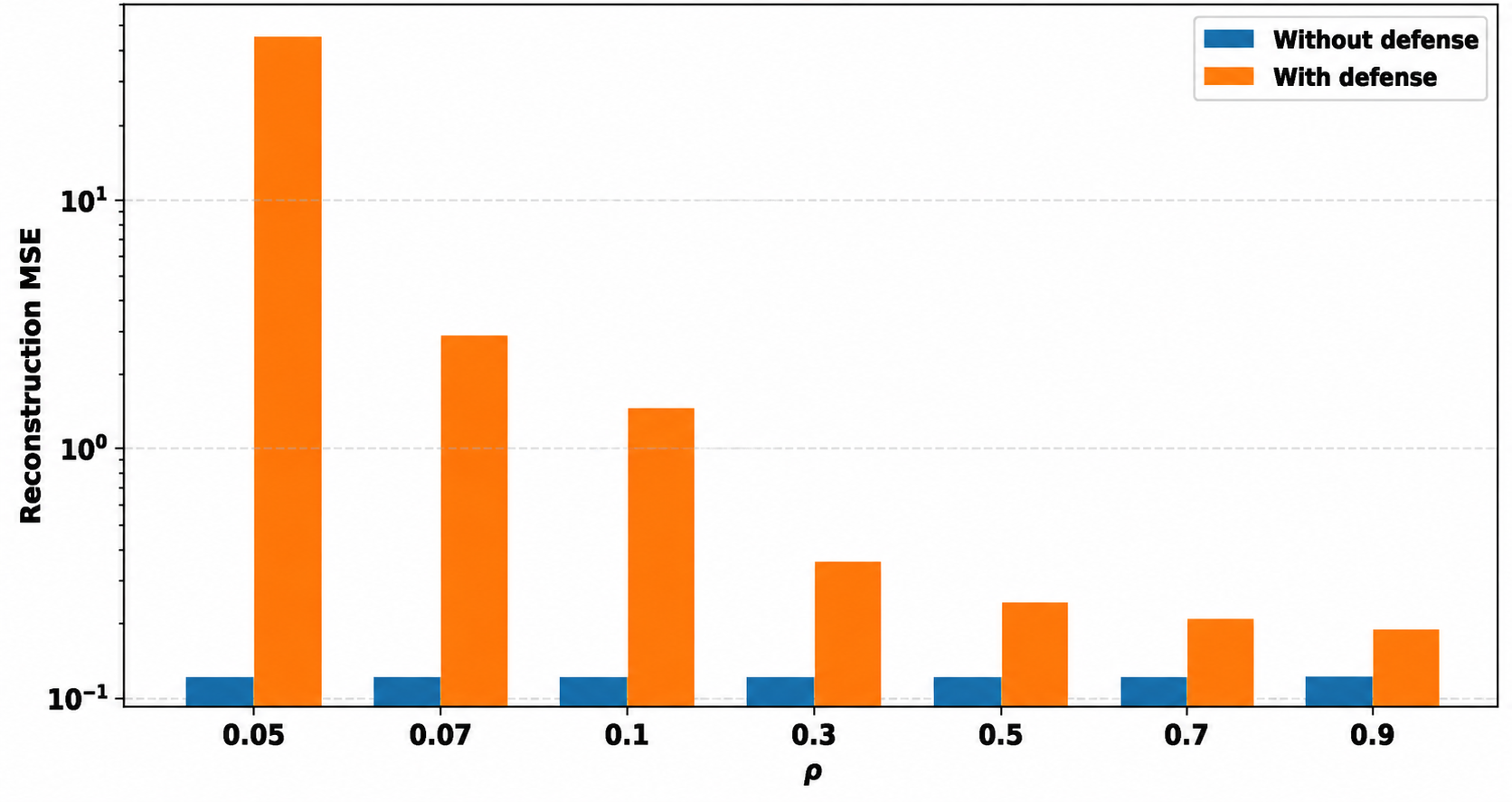}
    \end{minipage}\hfill
    \begin{minipage}[t]{0.45\textwidth}
        \centering
        \includegraphics[width=\linewidth]{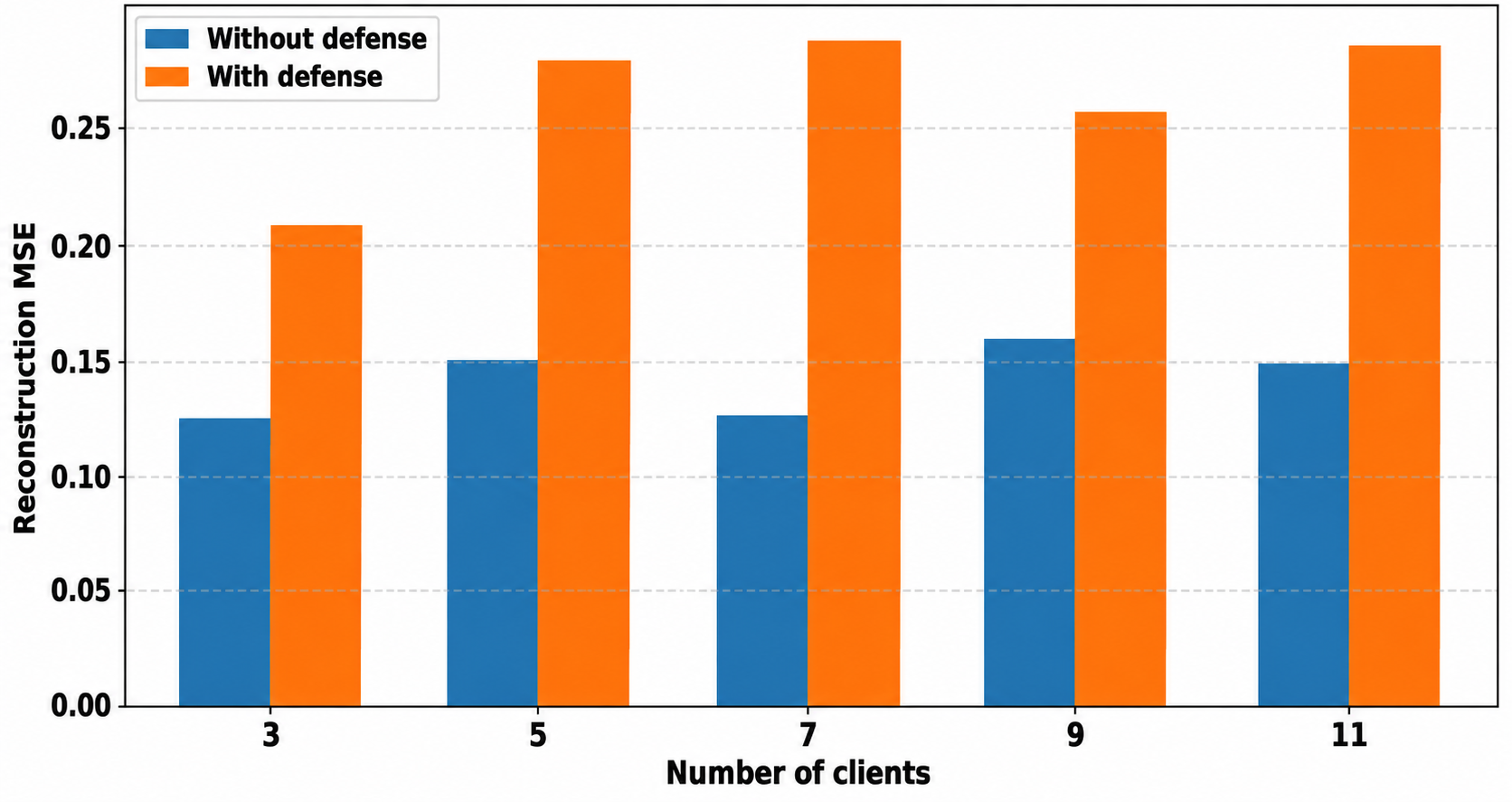}
    \end{minipage}
    \caption{
    Client-count and \(\rho\) ablation on additional datasets.
    \textbf{Left:} Reconstruction MSE on CIFAR-10 with and without
    \texttt{PRIVEE-U} as a function of \(\rho\). The undefended MSE remains
    nearly constant across \(\rho\), whereas the defended MSE is
    substantially higher for every tested setting and decreases as
    \(\rho\) increases, consistent with the perturbation scale
    \(\sigma=C/\rho\).
    \textbf{Right:} Reconstruction MSE on DRIVE with and without
    \texttt{PRIVEE-U} as the number of clients varies. The defended MSE remains
    consistently above the undefended MSE across all tested client
    counts, indicating that \texttt{PRIVEE-U} continues to provide meaningful
    protection as the federation grows.
    }
    \label{fig:other-datasets-ablation}
\end{figure*}

The CIFAR-10 results show that \texttt{PRIVEE-U} consistently increases
reconstruction error relative to the no-defense baseline for all tested
values of \(\rho\). Moreover, the defended MSE is largest for small
\(\rho\) and decreases monotonically as \(\rho\) increases, which is
consistent with the fact that smaller \(\rho\) yields a larger
perturbation scale \(\sigma=C/\rho\). The DRIVE results show that the
defense remains effective across all tested client counts: for every
configuration, the MSE with defense is clearly higher than the MSE
without defense. Although the absolute MSE varies somewhat with the
number of clients, the protective effect of \texttt{PRIVEE-U} is retained
throughout the tested range. Overall, these ablations support the same
qualitative conclusions as in the main paper: \(\rho\) provides a
predictable control over defense strength, and \texttt{PRIVEE-U}'s defense remains consistent with increasing number of clients.

\section{Our VFL Setting}

Figure~\ref{fig:VFL-schemas} illustrates the two-party VFL workflow used in
this study during both training and inference. In the training phase, the
active and passive parties first align records using shared entity identifiers
without exchanging their raw feature sets. Each party then applies its local
bottom model to its private features and transmits the resulting embeddings to
the federated coordinator. The coordinator concatenates the embeddings and
forwards the joint representation to the top model. Because the active party
holds the class labels, it computes the training loss and initiates
backpropagation. The resulting gradients are propagated through the
coordinator and returned to the corresponding parties, allowing the local and
top-model parameters to be updated without directly sharing raw features. In
the inference phase, the parties similarly compute and transmit local
embeddings, which the coordinator combines to produce the final prediction.

\begin{figure*}[!t]
    \centering
    \includegraphics[width=\textwidth]
    {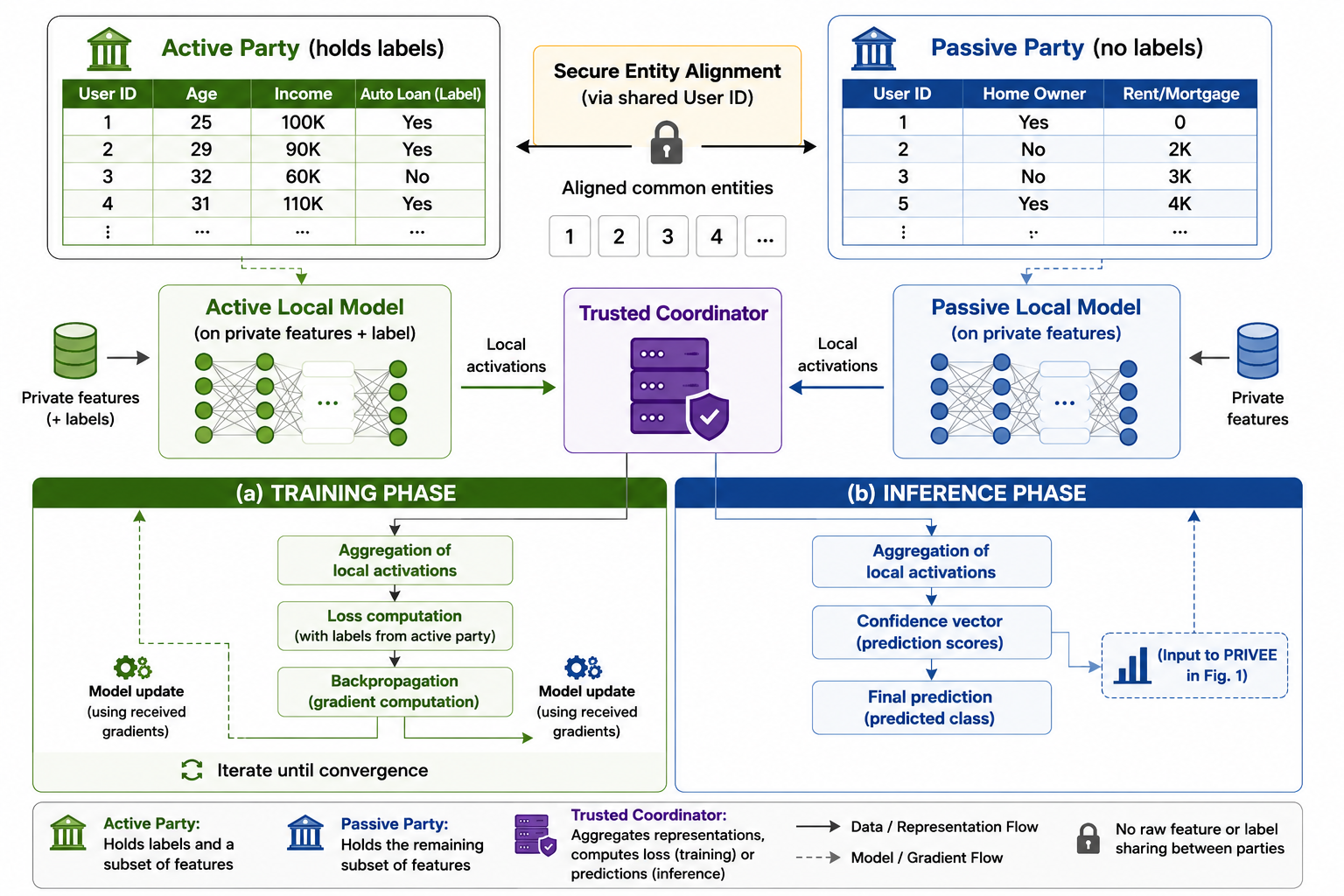}
    \caption{\textbf{Standard two-party vertical federated learning workflow during training and inference.}
    The active party holds a subset of private features and the class labels, whereas the passive party holds complementary private features for overlapping users. The parties first identify their common records through secure entity alignment using shared user identifiers. Each party then processes its private features through a local model and sends only the resulting local activations to a trusted coordinator. 
    \textbf{(a) Training phase:} the coordinator aggregates the local activations, computes the prediction loss using the labels provided by the active party, and performs backpropagation. The resulting gradients are returned to the corresponding parties to update their local models, and this process is repeated until convergence.
    \textbf{(b) Inference phase:} the trained local models independently compute activations for an aligned query instance. The coordinator aggregates these activations to produce the confidence vector and predicted class. The confidence vector constitutes the inference-time output protected by \texttt{PRIVEE} in the subsequent framework. Throughout both phases, raw features and labels are not directly exchanged between the participating parties.}
    \label{fig:VFL-schemas}
\end{figure*}

\section{Hyperparameter and Experimental Settings}
\label{sec:supp_hyperparameters}

During development, we tuned the learning rate and batch size using Optuna \cite{akiba2019optuna}. For each dataset, we conducted 100 optimization trials using the Tree-structured Parzen Estimator (TPE) sampler. The learning rate was sampled logarithmically from $10^{-5}$ to $10^{-1}$, while the batch size was selected from 32,64,128,256,512,1024,2048. After selecting the learning rate and batch size, we determined the number of training epochs separately for each dataset by examining the convergence of the model training and observing train, validation and test accuracies.

The attack settings were selected
based on attack convergence and reconstruction effectiveness, as measured
using the mean squared error (MSE). The exact candidate values, ranges, and
number of configurations evaluated were not retained.
Therefore, we report the final settings used in the experiments and the
criteria used to select them.

\subsection{Choice of Perturbation Amplitude Constant C}

To calibrate the base-amplitude constant $C$, we conducted a bivariate
grid search over
\begin{equation*}
C \in \left\{10^{-5},\,10^{-4},\,10^{-3},\,10^{-2},\,10^{-1},\,0.5,\,0.7,\,1.0\right\}
\end{equation*}
and perturbation intensities $\rho \in [0.05,0.90]$. Our objective was
to maximize the attacker's reconstruction MSE while obtaining a stable,
monotonically decreasing MSE trend as $\rho$ increased. Lower values of
$C$, particularly $C \leq 0.1$, produced a relatively flat MSE response
and provided insufficient baseline protection against GRNA. In contrast,
$C=1.0$ produced erratic changes in the resulting perturbation levels.
We therefore selected \textbf{$C=0.48$} as the base scaling constant. At this
setting, PRIVEE provides strong adversarial protection for small values
of $\rho$, while the protection strength decreases smoothly and
predictably as $\rho$ increases. This yields a reliable mechanism for
controlling the perturbation magnitude without degrading inference
accuracy.

\subsection{Model Setup}
\label{sec:model-setup}

We adopt the N-party VFL framework of~\cite{wei2022vertical}, involving one active and several passive parties. Both active and passive parties use a 128-unit hidden layer followed by a 64-unit output layer. The outputs are aggregated via direct concatenation. For the ResNet architecture, we use stacked blocks, each containing two 3×3 convolutional layers with batch normalization and ReLU activation. The LR models map input features directly to the output class space, with outputs from both parties summed to compute the final logits. 

\subsection{Dataset-Specific Training and Attack Settings}

Table~\ref{tab:dataset_hyperparameters} reports the final training and attack
settings for each dataset. Dataset-specific values override the default values
in the implementation.  The default GRNA learning rate was $0.01$, except for MNIST, for which the
dataset-specific learning rate was set to $0.1$. The CIFAR GIA implementation
uses default values of 100 attack iterations and a learning rate of
$10^{-3}$. The MNIST configuration overrides the generic GIA defaults by
using 500 iterations and a learning rate of $0.1$.

\begin{table*}[t]
    \centering
    \caption{Dataset-specific model-training and attack settings. A dash
    indicates that the corresponding attack was not used or that its setting
    was not specified for the dataset.}
    \label{tab:dataset_hyperparameters}
    \resizebox{\textwidth}{!}{%
    \begin{tabular}{lcccccccc}
        \toprule
        \textbf{Dataset}
        & \textbf{Batch Size}
        & \textbf{VFL Epochs}
        & \textbf{Classes}
        & \textbf{VFL LR}
        & \textbf{GRNA Epochs}
        & \textbf{GRNA LR}
        & \textbf{GIA Iterations}
        & \textbf{GIA LR} \\
        \midrule

        MNIST
        & 128
        & 60
        & 10
        & $5\times10^{-5}$
        & 50
        & $0.1$
        & 500
        & $0.1$ \\

        DRIVE
        & 128
        & 450
        & 11
        & $1\times10^{-4}$
        & 100
        & $0.01$
        & 500
        & 0.1 \\

        ADULT
        & 128
        & 100
        & 2
        & $1\times10^{-4}$
        & 100
        & $0.01$
        & 500
        & 0.1 \\

        CIFAR-10
        & 128
        & 30
        & 10
        & $2.085\times10^{-4}$
        & 10
        & $0.01$
        & 100
        & $1\times10^{-3}$ \\

        CIFAR-100
        & 512
        & 10
        & 100
        & $0.01$
        & 100
        & $0.01$
        & 100
        & $1\times10^{-3}$ \\
        \bottomrule
    \end{tabular}%
    }

    \vspace{2pt}
\end{table*}

\begin{table*}[!t]
    \centering
    \caption{Hyperparameter-selection procedure.}
    \label{tab:hyperparameter_selection}
    \resizebox{\textwidth}{!}{%
    \begin{tabular}{llll}
        \toprule
        \textbf{Hyperparameter}
        & \textbf{Selection Method}
        & \textbf{Final Settings}
        & \textbf{Selection Criterion} \\
        \midrule

        VFL learning rate
        & Optuna optimizer
        & Dataset-specific values in Table~\ref{tab:dataset_hyperparameters}
        & Stable convergence and predictive performance \\

        Attack learning rate
        & Implementation setting
        & Attack- and dataset-specific values in
          Table~\ref{tab:dataset_hyperparameters}
        & Stable attack convergence and reconstruction MSE \\

        Training batch size
        & Optuna optimizer
        & 128 or 512, depending on the dataset
        & Stable training and computational feasibility \\

        CIFAR GIA attack batch size
        & Implementation setting
        & 1
        & GPU-memory feasibility \\

        CIFAR GIA micro-batch size
        & Implementation setting
        & 1
        & Reduction of peak GPU-memory usage \\
        \bottomrule
    \end{tabular}%
    }
\end{table*}

\subsection{Shared Privacy and Federation Parameters}

Table~\ref{tab:shared_parameters} summarizes the settings shared across
datasets and experimental configurations.

\begin{table*}[t]
\centering
\footnotesize
\setlength{\tabcolsep}{6pt}
\renewcommand{\arraystretch}{1.15}
\caption{Shared privacy, federation, and optimization settings.}
\label{tab:shared_parameters}

\begin{tabular}{p{0.38\textwidth}p{0.56\textwidth}}
\hline
\textbf{Parameter (Value)} & \textbf{Description} \\
\hline

Number of organizations (2--$N$) &
One active organization and $\geq$2 passive organizations. \\

DP privacy budget $\varepsilon$ (GRN tables: $0.5$, $0.7$; GIA tables: $0.5$, $1$) &
The DP baseline is swept at two budgets per attack family. \\

Differential privacy parameter $\delta$ ($1\times10^{-5}$) &
Failure probability used by the Gaussian mechanism. \\

Sensitivity (0.1) &
Sensitivity used to compute the Gaussian noise scale. \\

Default GRNA learning rate (0.01) &
Used when no dataset-specific GRNA learning rate is provided. \\

Optimizer (Adam) &
Optimizes the reconstructed passive-party inputs in GIA and GRNA. \\

Attack loss (MSE) &
Measures the discrepancy between generated and released confidence vectors. \\

Feature partition (Dataset-dependent) &
Features are partitioned between the active and passive organizations. \\
\hline
\end{tabular}
\end{table*}
\subsection{Model Architecture Settings}

The model architectures used for MNIST, DRIVE, ADULT, CIFAR-10, and CIFAR-100
are summarized in Table~\ref{tab:model_architectures}.

The hidden-layer widths and output embedding dimensions of the fully connected active- and passive-party models are fixed architectural settings, not tuned per dataset: both parties use a 128-unit hidden layer followed by a 64-unit output layer (Section~\ref{sec:model-setup}), as reported in Table~\ref{tab:model_architectures}.

\subsection{Attack Implementation Settings}

Table~\ref{tab:attack_settings} reports the settings directly encoded in the
GIA and GRNA implementations.

\begin{table*}[t]
    \centering
    \footnotesize 
    \caption{Implementation settings for the gradient inversion attack (GIA)
    and generative regression network attack (GRNA).}
    \label{tab:attack_settings}
    \begin{tabular}{lll}
        \toprule
        \textbf{Attack}
        & \textbf{Parameter}
        & \textbf{Final Setting} \\
        \midrule

        \multirow{8}{*}{Generic GIA}
        & Reconstructed passive input initialization
        & All zeros \\
        & Optimized variable
        & Reconstructed passive-party input \\
        & Optimizer
        & Adam \\
        & Objective
        & MSE between predicted and target confidence vectors \\
        & Default learning rate
        & $1\times10^{-3}$ \\
        & Default number of iterations
        & 500 \\
        & Reconstruction bounds
        & $[0,1]$ \\
        & Progress-reporting interval
        & Every 100 iterations \\
        \midrule

        \multirow{10}{*}{CIFAR GIA}
        & Reconstructed passive input initialization
        & All zeros \\
        & Optimized variable
        & Reconstructed passive-party image \\
        & Optimizer
        & Adam \\
        & Objective
        & MSE between predicted and target confidence vectors \\
        & Default learning rate
        & $1\times10^{-3}$ \\
        & Default number of iterations
        & 100 \\
        & Attack batch size
        & 1 \\
        & Passive-model micro-batch size
        & 1 \\
        & Reconstruction bounds
        & $[0,1]$ \\
        & Numerical precision
        & FP16 automatic mixed precision \\
        \midrule

        \multirow{8}{*}{GRNA}
        & Reconstructed passive input initialization
        & Samples from a standard normal distribution \\
        & Optimized variable
        & Reconstructed passive-party training data \\
        & Optimizer
        & Adam \\
        & Objective
        & MSE between generated and target confidence vectors \\
        & Default learning rate
        & $0.01$ in the experimental configuration \\
        & Model operating mode
        & Evaluation mode \\
        & Model parameter updates
        & Disabled; all trained model parameters are frozen \\
        & Optimization scope
        & Reconstructed passive-party inputs only \\
        \bottomrule
    \end{tabular}%
\end{table*}
For the generic GIA implementation, the passive-party input estimate is initialized to zero and clipped to $[0,1]$ after each optimization step. For GRNA, it is initialized with independent standard normal samples. During both attacks, trained VFL model parameters remain fixed, and only the reconstructed passive-party inputs are optimized.

In the standard GRNA implementation, mini-batch losses are averaged before one optimization step per epoch. In the CIFAR-specific implementation, one optimization step is performed per mini-batch, with the reported epoch loss computed as the average mini-batch loss.

\subsection{Hyperparameter-Selection Procedure}

Table~\ref{tab:hyperparameter_selection} summarizes the available information
about the hyperparameter-selection process.  Several candidate learning-rate and batch-size values were evaluated during
development.  However, for attack settings, the exact number of values and search ranges are not retained. Consequently, the paper reports the final
parameter settings and the criteria used to select them.

\subsection{Data Preprocessing}
All code required for data preprocessing is included in the supplementary code repository. In particular, the \texttt{create\_dataset.py} script implements the dataset preparation procedures for MNIST, CIFAR-10, CIFAR-100, DRIVE, and ADULT. The script loads the corresponding raw data, performs the required dataset-specific preprocessing and formatting, and generates the processed files used by the experimental pipeline. It can be executed using \texttt{python create\_dataset.py --dataset DATASET}, where \texttt{DATASET} is one of \texttt{MNIST}, \texttt{CIFAR10}, \texttt{CIFAR100}, \texttt{DRIVE}, or \texttt{ADULT}. Any additional transformations applied when loading the processed datasets are also implemented in the released experimental code.

\paragraph{Code Availability and Documentation.} All source code required to conduct and analyze the experiments is included as supplement in the ``code and Data Supplement'' section. The repository contains the implementations of the proposed methods, baseline defenses, attacks, data-preprocessing procedures, model-training pipelines, evaluation routines, and scripts used to generate the reported results. Upon publication, the complete source code will be released publicly under a license permitting free use for research purposes. The implementations of the new methods include comments describing the main computational steps and identifying the corresponding algorithms, equations, or methodological components presented in the paper.
\begin{table*}[t]
\centering
\footnotesize
\setlength{\tabcolsep}{4pt}
\renewcommand{\arraystretch}{1.1}
\caption{Model architecture settings.}
\label{tab:model_architectures}

\begin{tabular}{p{0.3\textwidth}p{0.37\textwidth}p{0.25\textwidth}}
\hline
\textbf{Component} &
\textbf{Setting} &
\textbf{Value} \\
\hline

\multirow{4}{=}{MNIST/DRIVE active-party model}
& Number of fully connected layers & 2 \\
& Hidden-layer width & 128 \\
& Output embedding dimension & 64 \\
& Hidden activation & ReLU \\
\hline

\multirow{4}{=}{MNIST/DRIVE passive-party model}
& Number of fully connected layers & 2 \\
& Hidden-layer width & 128 \\
& Output embedding dimension & 64 \\
& Hidden activation & ReLU \\
\hline

\multirow{2}{=}{Two-party learning coordinator}
& Number of fully connected layers & 1 \\
& Output activation & Softmax for confidence-score release \\
\hline

\multirow{3}{=}{Multi-party learning coordinator}
& Number of fully connected layers & 2 \\
& Hidden-layer width & 256 \\
& Hidden activation & ReLU \\
\hline

\multirow{9}{=}{CIFAR bottom model}
& Initial convolution channels & 16 \\
& Initial convolution kernel & $3\times3$ \\
& Initial convolution stride & 1 \\
& Initial convolution padding & 1 \\
& Number of residual stages & 3 \\
& Residual blocks per stage & 3 \\
& Stage output channels & 16, 32, and 64 \\
& Stage strides & 1, 2, and 2 \\
& Pooling operation & Adaptive global average pooling \\
\hline

\multirow{2}{=}{CIFAR bottom model output}
& Default embedding dimension & 64 \\
& Final projection & Fully connected layer \\
\hline

\multirow{2}{=}{CIFAR learning coordinator}
& Number of fully connected layers & 1 \\
& Output & Class logits \\
\hline

\multirow{4}{=}{ADULT active-party model}
& Number of fully connected layers & 2 \\
& Hidden-layer width & 128 \\
& Output embedding dimension & 64 \\
& Hidden activation & ReLU \\
\hline

\multirow{4}{=}{ADULT passive-party model}
& Number of fully connected layers & 2 \\
& Hidden-layer width & 128 \\
& Output embedding dimension & 64 \\
& Hidden activation & ReLU \\
\hline

\multirow{2}{=}{ADULT learning coordinator}
& Number of fully connected layers & 1 \\
& Output & Class logits \\
\hline

\multirow{2}{=}{VFL logistic model}
& Bias terms & Disabled \\
& Output activation & Softmax \\
\hline

\end{tabular}
\end{table*}

\paragraph{Random Seed Configuration.}
All experiments involving random initialization, data shuffling, feature
partitioning, or stochastic perturbation were conducted using fixed random
seeds. Before each experimental run, the same seed was assigned to Python's
\texttt{random} module, NumPy, and PyTorch, including all available CUDA
devices. The PyTorch deterministic-execution settings were also enabled where
supported. Specifically, the seeds were initialized using \texttt{random.seed(SEED)}, \texttt{numpy.random.seed(SEED)},
\texttt{torch.manual\_seed(SEED)}, and
\texttt{torch.cuda.manual\_seed\_all(SEED)}. The value of \texttt{SEED} used
for each reported experiment is provided in the released experimental
configuration and execution scripts, allowing the reported results to be
replicated. We have used 10 algorithmic runs for each of our results with different random seeds.

\subsection{Computing Infrastructure.}
Experiments were conducted on a shared Linux-based high-performance computing cluster using NVIDIA GPUs. Depending on availability, jobs ran on NVIDIA T4 or V100 GPUs (16~GB), NVIDIA A30 GPUs (24~GB), or NVIDIA L40S GPUs (48~GB).

\bibliography{aaai2027}

\end{document}